\documentclass{article}
\usepackage{subcaption}
\usepackage{booktabs}
\usepackage{enumitem}
\usepackage{multirow}
\usepackage{pifont}
\usepackage{colortbl}
\usepackage{xcolor}
\definecolor{lightgray}{gray}{0.9}
\usepackage{tikz}
\usepackage{algorithm}
\usepackage{algorithmic}
\usepackage{wrapfig}
\usepackage{caption}
\usepackage{float}

\usepackage{bm}
\usepackage{amsmath}
\usepackage{amssymb}
\usepackage{mathtools}
\usepackage{amsthm}
\usepackage{soul}
\usepackage{natbib}

\def\thmspace{0.2em}
\newtheorem{theorem}{\hspace{\thmspace}{\bf Theorem}\!}

\newtheorem{proposition}{\hspace{\thmspace}{\bf Proposition}\!}

\newtheorem{remark}{\hspace{\thmspace}{\bf Remark}\!}

\usepackage{PRIMEarxiv}

\usepackage[utf8]{inputenc} 
\usepackage[T1]{fontenc}    
\usepackage{hyperref}       
\usepackage{url}            
\usepackage{booktabs}       
\usepackage{amsfonts}       
\usepackage{nicefrac}       
\usepackage{microtype}      
\usepackage{lipsum}
\usepackage{fancyhdr}       
\usepackage{graphicx}       
\graphicspath{{media/}}     

\title{Discrete-Guided Diffusion for Scalable and Safe Multi-Robot Motion Planning
}

\author{
  Jinhao Liang \\
  University of Virginia \\
  \texttt{jliang@email.virginia.edu} \\
  \And
  Sven Koenig \\
  University of California, Irvine \\
  \texttt{sven.koenig@uci.edu} \\
  \And
  Ferdinando Fioretto \\
  University of Virginia \\
  \texttt{fioretto@virginia.edu} \\
}

\begin{document}
\maketitle

\begin{abstract}
Multi-Robot Motion Planning (MRMP) involves generating collision-free trajectories for multiple robots operating in a shared continuous workspace. While discrete multi-agent path finding (MAPF) methods are broadly adopted due to their scalability, their coarse discretization severely limits trajectory quality. In contrast, continuous optimization-based planners offer higher-quality paths but suffer from the curse of dimensionality, resulting in poor scalability with respect to the number of robots. This paper tackles the limitations of these two approaches by introducing a novel framework that integrates discrete MAPF solvers with constrained generative diffusion models. 
The resulting framework, called \emph{Discrete-Guided Diffusion} (DGD), has three key characteristics: (1) it decomposes the original nonconvex MRMP problem into tractable subproblems with convex configuration spaces, (2) it combines discrete MAPF solutions with constrained optimization techniques to guide diffusion models capture complex spatiotemporal dependencies among robots, and (3) it incorporates a lightweight constraint repair mechanism to ensure trajectory feasibility. The proposed method sets a new state-of-the-art performance in large-scale, complex environments, scaling to 100 robots while achieving planning efficiency and high success rates.
\end{abstract}


\section{Introduction}
Multi-Robot Motion Planning (MRMP) is a fundamental problem in robotics that requires generating collision-free trajectories for multiple robots operating in a shared environment. MRMP arises in diverse applications such as automated warehouses, coordinated drone fleets, and autonomous driving. Despite its importance, efficiently solving MRMP in complex environments remains a significant challenge due to the high dimensionality and combinatorial complexity~\cite{yu2013structure}.

To address these challenges, two major paradigms have emerged. Optimization-based methods formulate the problem as a continuous, often nonconvex trajectory optimization problem and are capable of producing smooth, high-quality paths~\cite{doi:10.1126/scirobotics.adf7843}. However, their scalability is severely limited as the number of robots and obstacles increases. 
Another body of work considers a discretized version of MRMP, known as Multi-Agent Path Finding (MAPF). It discretizes the space and time, significantly reducing computational complexity~\cite{li2021eecbs,okumura2024engineering}. MAPF algorithms scale to hundreds of robots, but their reliance on discretized grid-based movement and synchronized time steps limits their applicability in continuous and dynamic environments.

To address these limitations, recent work has explored the use of generative models, particularly diffusion models, to learn distributions over trajectories in continuous spaces~\cite{xiao2022motion,carvalho2023motion}. These models show promising results in single-robot planning as they enable diverse and high-quality trajectory generation. However, their extension to multi-robot settings introduces significant challenges: \emph{Diffusion models must capture complex inter-robot spatiotemporal dependencies while simultaneously avoiding collisions}, a problem that becomes highly intractable as the number of robots and obstacles grows.

Several recent approaches attempt to extend diffusion models to MRMP via gradient-based guidance~\cite{shaoul2024multi,ding2025swarmdiff}. These methods are able to generate high quality trajectories, but are limited by two key challenges: First, the difficulty of ensuring the feasibility of the generated trajectories with respect to non-collision and kinodynamic constraints. This is hard since gradient-based guidance cannot natively guarantee global constraint satisfaction. Second, these models struggle in cluttered or dense environments, where the complexity of the configuration space increases. Recently, ~\citet{liang2025simultaneous} proposed an approach to ensure feasibility through repeated projections within the diffusion process. This is a promising solution, but it comes at a high computational overhead, due to the nonconvexity and high dimensionality of constraints.

\begin{figure*}[!t]
  \centering
    \includegraphics[width=0.8\textwidth]{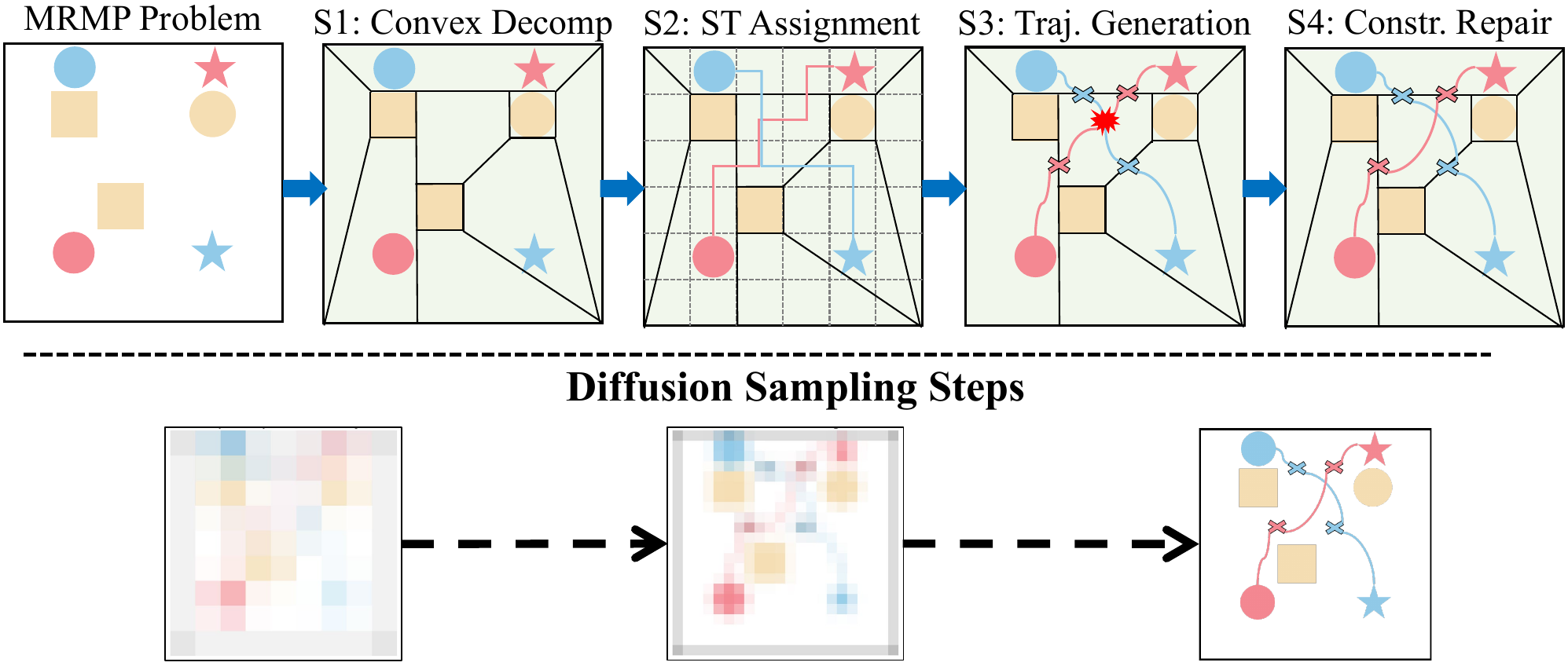}
  \caption{Overview of Discrete-Guided Diffusion. \textbf{S1} shows that the nonconvex configuration space is approximated by multiple convex regions. The wheat regions denote the obstacles, and the green regions denote the free space. The robots move from their start $\bullet$ to their goal positions $\star$. \textbf{S2} illustrates that the local start and goal for each robot in each subproblem is determined by the spatiotemporal dependency obtained from MAPF solutions. \textbf{S3} shows that the trajectory is generated by diffusion models guided by discrete spatiotemporal dependency. Collision is marked by \ding{88}. \textbf{S4} shows that the infeasible problem is repaired by constraint-aware diffusion models.}
  \label{fig:overview}
\end{figure*}

\noindent\textbf{Contribution.}
This paper addresses these challenges through a novel integration of discrete MAPF with continuous denoising diffusion. 
The proposed \emph{Discrete-Guided Diffusion} (DGD), first decomposes the original MRMP problem into a sequence of local subproblems with convex configuration space, each defined by the spatiotemporal structure obtained from the MAPF solution. 
On each subproblem, a diffusion model is then used to generate trajectories guided by this spatiotemporal structure. To ensure feasibility, DGD relies on a lightweight projected guidance rendering each step of the diffusion process feasible, and enforced only when constraint violations are detected. The whole framework is schematically illustrated in Figure~\ref{fig:overview}. 
The effectiveness and efficiency of DGD is demonstrated both theoretically and practically across a range of challenging benchmarks with up to 100+ robots.

\section{Related Work}
In this section, we review three research areas that form the foundation for our work: multi-robot motion planning (MRMP), discrete multi-agent path finding (MAPF), and generative models for motion planning.
We summarize the key methods and challenges in each area, highlighting the scalability, constraint handling, and feasibility guarantees that motivate our proposed approach.

\paragraph{Multi-robot Motion Planning.}
MRMP methods address trajectory generation for multiple robots in continuous space. 
There are two main approaches for MRMP:
{\bf (1)} Sampling-based algorithm, where constructs feasible trajectories by randomly sampling the state space~\cite{gammell2014informed,shome2020drrt}. While sampling-based planners are probabilistically complete, they require dense sampling and struggle to enforce continuous kinodynamic constraints, limiting their practical applications. 
{\bf (2)} Optimization-based planner, where the planning problem is formulated as an optimization program, such as a sequential convex program~\cite{6385823,9197162} and a mixed-integer program~\cite{6225009}. While optimization-based methods can handle various constraints, extending them to the multi-robot case couples all robots into a single high-dimensional optimization problem, suffering from exponential growth in decision variables and local minima. Recently, Graph of Convex Sets (GCS) has been proposed to mitigate this issue for single robot motion planning by decomposing the environment into a collection of convex regions~\cite{doi:10.1126/scirobotics.adf7843,chia2024gcs}. However, applying GCS to MRMP is quite challenging due to the spatiotemporal dependencies among multiple robots.

\paragraph{Discrete Multi-agent Path Finding.}
Multi-agent path finding (MAPF) focuses on a discretized version of MRMP by discretizing both time and space into steps and grids, respectively~\cite{stern2019multi}.
The MAPF literature has produced highly scalable algorithms such as Conflict-Based Search (CBS) and its bounded-suboptimal variants~\cite{sharon2015conflict,li2021eecbs,okumura2024engineering}.
These planners can coordinate hundreds of agents with correctness and solution quality guarantees in grid worlds by resolving pairwise collisions in a low-dimensional discrete space and have been adopted in large automated-warehouse deployments.
However, their reliance on synchronized, axis-aligned motions and fixed timesteps makes it difficult to transfer the resulting plans to real robots with continuous dynamics, kinodynamic limits, or tight clearances; post-processing heuristics such as temporal smoothing or continuous re-planning only partially mitigate this gap~\cite{honig2018trajectory}.

\paragraph{Generative models for motion planning.}
Diffusion generative models have recently emerged as a powerful alternative to sampling-based and optimization planners, learning a distribution over continuous trajectories that captures multimodality without hand-tuned costs~\cite{xiao2022motion,carvalho2023motion,luo2024potential}.
While early work focused on single-robot problems, extending to multi-robot motion planning requires enforcing \emph{joint} collision and kinodynamic constraints, which dramatically increases the complexity of the diffusion process. 

To address this issue, the recent literature has explored three main strategies: 
(i) \emph{Classifier or gradient guidance} adds an external collision penalty during the reverse process~\cite{ding2025swarmdiff}. Although simple, guidance offers no guarantees on constrained satisfaction and degrades in dense environments. 
\citet{shaoul2024multi} extends this approach by using diffusion models as the single-robot planner in a constraint-based MAPF framework, but this process utilizes diffusion to generate single-robot trajectories without explicitly modeling inter-robot interactions.
(ii) \emph{Constrained score estimation} methods learn the score function with added penalty terms~\cite{naderiparizi2025constrained}, but constraints satisfaction still cannot be guaranteed.
(iii) Finally, \emph{Projection-based refinement} methods, \citet{liang2025simultaneous}, interleaves diffusion steps with a nonlinear projection onto the global feasible set.
This approach produces provably collision-free trajectories for a few robots and is therefore central to our discussion and later comparisons. Nonetheless, the need to solve a full high-dimensional projection at every diffusion step incurs substantial computational overhead and limits scalability to larger teams.

To address these limitations, the proposed \emph{Discrete-Guided Diffusion} (DGD) exploits a hybrid approach where a MAPF-derived spatiotemporal skeleton is fused with a constraint-aware diffusion process \emph{inside} locally convex subproblems, avoiding Liang et al.’s costly global projections while preserving formal feasibility and yielding state-of-the-art performance on large-scale MRMP benchmarks.

\section{Preliminaries}

\paragraph{Multi-Robot Motion Planning (MRMP).} 
MRMP involves computing collision-free trajectories for multiple robots navigating a shared environment from designated start positions to goal states. 
Let $\mathcal{A}=\{a_1,\dots,a_{N_a}\}$ be a set of $N_a$ robots that move
in a bounded, planar workspace $\mathcal{W}\subset\mathbb{R}^2$ containing
a set of static obstacles $\mathcal{O}=\{o_1,\dots,o_{N_o}\}$.
The \emph{free configuration space} is defined as the set of all points in the workspace that are not occupied by obstacles:
\(
   \mathcal{C}_{\mathrm f}
   := \mathcal{W}\setminus
      \bigl(\bigcup_{o\in\mathcal{O}} o\bigr).
\)

Each robot $a_i$ is modeled as a disk of radius $r_i>0$ and is characterized at discrete time steps $h\in\{0,\ldots,H\}$ by a 2-D position $\pi_i^{h} = (x_i^{h},y_i^{h})\in\mathcal{C}_{\mathrm f}.$ A \emph{trajectory} for robot $a_i$ is the sequence $\bm{\pi}_i := (\pi_i^{0},\pi_i^{1},\dots,\pi_i^{H}).$
For each robot, their start and target states are defined by sets $\bm{B} = [b_1, b_2, \ldots, b_{N_a}]$ and $\bm{E} = [e_1, e_2, \ldots, e_{N_a}]$. 

The goal of MRMP is to compute a set of trajectories $\boldsymbol{\Pi} = \{\boldsymbol{\pi}_1, \boldsymbol{\pi}_2, \ldots, \boldsymbol{\pi}_{N_a}\}$ such that each robot starts at its designated start position and reaches its target position while avoiding collisions with obstacles and other robots.
For every time index $h$ we require
\begin{subequations}
\begin{align}
  \pi_i^{h} &\in \mathcal{C}_{\mathrm f},
  &\forall i,
  \label{eq:obst-avoid}\\[2pt]
  \|\pi_i^{h}-\pi_j^{h}\|_2 &\ge r_i+r_j,
  &\forall i<j,          \label{eq:robot-avoid}\\[2pt]
  \|\pi_i^{h+1}-\pi_i^{h}\|_2 &\le v_{\max}\,\Delta t,
  &\forall i,h<H,        \label{eq:dyn}
\end{align}
\end{subequations}
where~\eqref{eq:obst-avoid} enforces obstacle avoidance,
\eqref{eq:robot-avoid} inter-robot separation, and \eqref{eq:dyn} a
first-order kinematic bound with time step~$\Delta t$ and speed
limit~$v_{\max}$.

\paragraph{Multi-Agent Path Finding (MAPF).}
\label{sec:prelim-mapf}
In contrast, Multi-Agent Path Finding assumes the \emph{configuration space and time to be discretized}, which significantly reduces the complexity of the original problem. The workspace is represented by an undirected graph $\mathcal G=(V,E)$ whose vertices encode grid cells or roadmap milestones.  
The input to a MAPF problem is a tuple $\langle \mathcal{G}, \mathrm{S}, \mathrm{T}\rangle$, where  $\mathrm{S}\!:\![1,\ldots,k]\rightarrow V$ maps robots to their source vertices and $\mathrm{T}\!:\![1,\ldots,k]\rightarrow V$ maps robots to their target vertices. 
At each integer time step~$t$, robot~$i$ occupies a vertex $v_i^{t}\!\in\!V$ and chooses either a wait action or an edge $(v_i^{t},v_i^{t+1})\in E$.

The result of a MAPF algorithm is a set of discrete multi-robot trajectories $\boldsymbol{\Pi}_{M} = \{\boldsymbol{\pi}_{m,1}, \boldsymbol{\pi}_{m,2}, \ldots, \boldsymbol{\pi}_{m,N_a}\}$, specifying the grid location of each robot at every time step, satisfying vertex conflicts $v_i^t \ne v_j^t \ \forall i \ne j,\, t$ and edge conflicts $(v_i^t, v_i^{t+1}) \ne (v_j^{t+1}, v_j^t) \ \forall i \ne j,\, t$.

\subsection*{The Discrete-Continuous Gap} 
The MAPF solution $\boldsymbol{\Pi}_{M}$ satisfies the constraints of the discrete problem, but it does not directly translate to a continuous trajectory for the robots. To obtain a continuous trajectory, three key challenges arise:
\begin{enumerate}[leftmargin=*,parsep=0pt]
\item Embedding the vertices $V$ of $\mathcal{G}$ in the continuous space $\mathcal{C}_{\mathrm f}$, which may not be a grid.
\item Time-parameterizing each edge into a collision-free segment that respects constraint~\eqref{eq:dyn}. This requires ensuring that the trajectory segments do not violate the kinematic limits of the robots.
\item Although the MAPF solution ensures discrete collision avoidance, it does not exploit constraint~\eqref{eq:robot-avoid} in the continuous space, which can result in suboptimal solutions.
\end{enumerate}

A coarse grid may also lead to unrealistic motions or infeasible timing, whereas a fine grid blows up the state space exponentially, making MAPF itself expensive.  
To contrast these challenges, this paper proposes a framework that leverages MAPF solely as a \emph{structural guide} while a continuous diffusion model refines the trajectories to respect~\eqref{eq:obst-avoid}--\eqref{eq:dyn}.

\section{Discrete-Guided Diffusion Framework}
\label{sec:dgd-framework}
Solving the MRMP problem poses two fundamental challenges. First, MRMP is \emph{NP-hard} and the complexity grows exponentially with the number of robots, making it difficult to scale to large problems while preserving high-quality trajectories generation. Second, although generative models offer better scalability and trajectory quality, ensuring the \emph{feasibility} of entire trajectories remains challenging, particularly in high-dimensional planning spaces.

To address these challenges, this section presents an MRMP framework that couples a discrete MAPF backbone with a continuous score-based diffusion model to generate high-quality trajectories for multiple robots. The framework, called \emph{Discrete-Guided Diffusion} (DGD). Specifically, the first two stages mitigate the scalability challenge, and the next two stages address the challenge of constraint satisfaction and trajectory quality in high-dimensional spaces:
\begin{enumerate}[label=\textbf{S\arabic*}, nosep]
  \item \emph{Priority-based Convex Decomposition} (PBD):  approximation of the nonconvex free configuration space~$\mathcal C_{\mathrm f}$ using nonoverlapping convex regions to enable problem decomposition (Section~\ref{sec:pbd})
  \item \emph{Spatiotemporal assignment}: determination of entry and exit events for each convex region by leveraging spatiotemporal dependencies extracted from a MAPF solution~$\bm{\Pi}_M$ (Section~\ref{sec:sta}).
  \item \emph{Diffusion-based Trajectory Generation}: generation of high-quality trajectories for each subproblem using a diffusion model guided by~$\bm{\Pi}_M$ (Section~\ref{sec:diffusion}).
  \item \emph{Constraint-aware Diffusion Refinement}: correction of infeasible trajectories through constraint-aware diffusion models by introducing a lightweight projection mechanism to enforce constraint satisfaction (Section~\ref{sec:repair}).
\end{enumerate}

Proofs for the theoretical results discussed in this section are referred to Appendix~\ref{app: missing_proof}.

\begin{figure}[t]
  \centering
  \begin{subfigure}[t]{0.3\columnwidth}
    \centering \includegraphics[width=\textwidth]{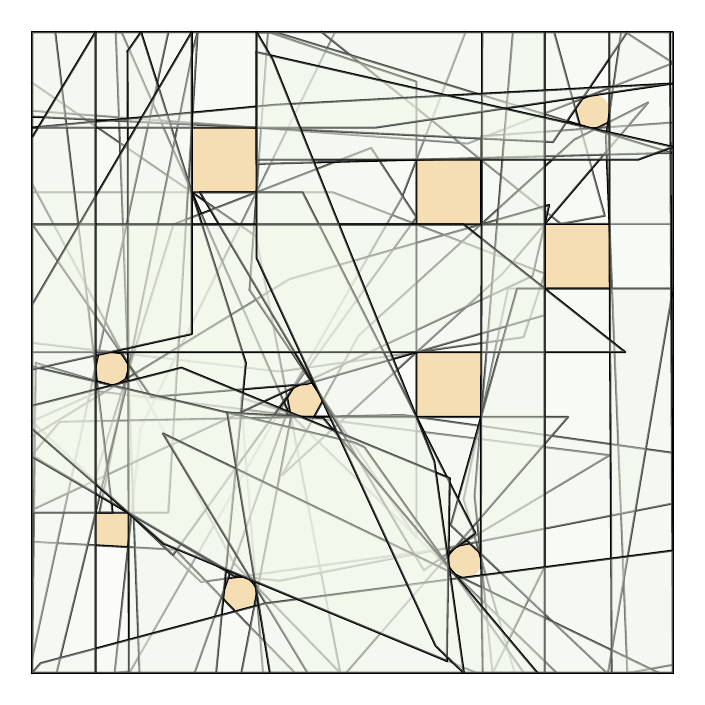}
    \caption{VCC} \label{fig:VCC}
  \end{subfigure}
  \begin{subfigure}[t]{0.3\columnwidth}
    \centering \includegraphics[width=\textwidth]{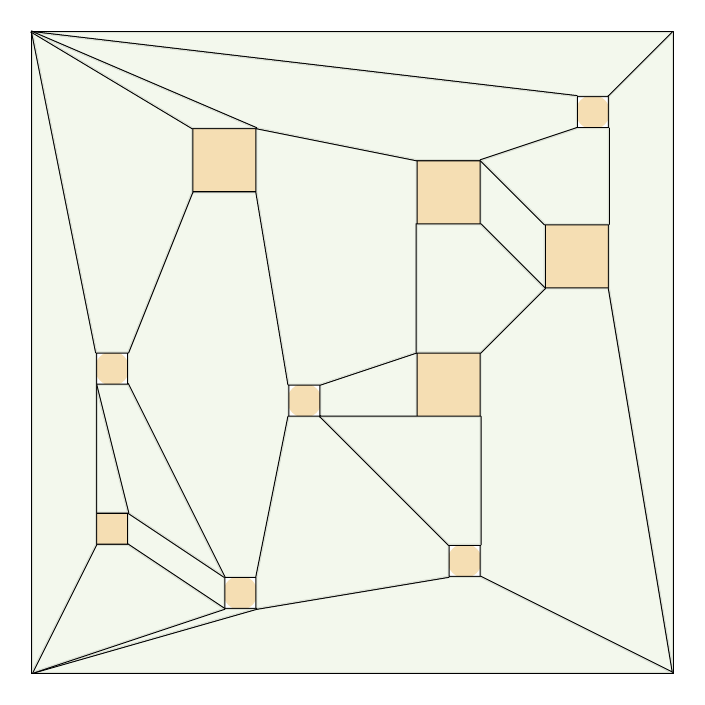}
    \caption{PBD} \label{fig:PBD}
  \end{subfigure}
  \caption{Comparison of Configuration Space Decomposition results between Visibility Clique Cover (VCC) and the proposed priority-based decomposition (PBD). Yellow objects represent obstacles, and each green-filled polygon with a black border indicates a convex region generated by the decomposition methods.}
  \label{fig:convex decomposition}
\end{figure}

\subsection{Priority-Based Convex Decomposition}
\label{sec:pbd}

Recent advancements in single-robot planning replace the original nonconvex configuration space $\mathcal C_{\mathrm f}$ with a \emph{union of convex polytopes}, enabling single-robot motion planning to be formulated as convex programs~\cite{werner2024approximating,werner2024faster}.
The current state-of-the-art, the Visibility Clique Cover (VCC) algorithm~\cite{werner2024approximating}, combines clique-based sampling with IRIS-style region inflation to approximate the free configuration space using a small number of large convex polytopes.
While this eliminates nonconvex obstacle avoidance constraints, the resulting polytopes \emph{overlap} heavily in practice. As shown in Figure~\ref{fig:VCC}, VCC produces substantially more sets than required to cover $\mathcal C_{\mathrm f}$, highlighting the inefficiency introduced by overlapping regions. For MRMP, this is problematic for two reasons:
{\bf (i)}~First, overlap explodes the number of discrete mode switches that must be co-optimized, and 
{\bf (ii)}~Second, a single robot may lie in \emph{several} polytopes at once, making a clean decomposition into independent subproblems impossible. \emph{This motivates us to find more compact and structure-aware representations.}

\paragraph{Priority-based Decomposition (PBD).} To address these limitations, we introduce a novel decomposition approach that guarantees non-overlapping convex regions while optimizing for robot traffic patterns. This method builds upon computational geometry techniques, inspired by the Hertel-Mehlhorn algorithm~\cite{hertel1983fast}.  It produces a partition $\mathcal{C}_{\text{f}}^{\text{c}} = \{R_1, R_2, \ldots, R_k\}$ where $R_i \cap R_j = \emptyset$ for $i \neq j$ and $\bigcup_{i=1}^k R_i \in \mathcal{C}_{\text{f}}$, ensuring each robot belongs to \emph{exactly one region at any time}. 

\begin{wrapfigure}[17]{r}{0.53\columnwidth} 
\vspace{-12pt}
    \vspace{-\intextsep} 
    \begin{minipage}{\linewidth} 
        \begin{algorithm}[H]
        \small
        \caption{Priority-based Decomposition (PBD)}
        \label{alg:Priority-based-Decomposition}
        \begin{algorithmic}[1]
        \STATE \textbf{Input:} Configuration space $\mathcal{C}_{\text{f}}$, MAPF solution $\boldsymbol{\Pi}_{M}$
        \STATE \textbf{Output:} Non-overlapping convex partition $\mathcal{C}_{\text{f}}^{\text{c}}$
        \STATE $R \leftarrow \text{RemoveHoles}(\mathcal{C}_{\text{f}})$ \COMMENT{Create Simple Polygon}
        \STATE $R \leftarrow \text{InitialRegions}(R)$ \COMMENT{Triangulation}
        \STATE Compute priority $p(r)$ for each $r \in R$
        \STATE Build adjacency graph ${Adj}[r]$ for each $r \in R$
        \STATE $\mathcal{H} \leftarrow \text{CreateMaxHeap}(R, \text{Adj}, p)$
        \WHILE{$\mathcal{H}$ not empty}
            \STATE $(r_1, r_2) \leftarrow \mathcal{H}.\text{pop}()$
            \IF{$(r_1 \notin R) \lor (r_2 \notin R) \lor \text{IsNotConvex}(r_1 \cup r_2)$}
                \STATE \textbf{continue}
            \ENDIF
            \STATE $r_{\text{new}} \leftarrow r_1 \cup r_2$
            \STATE Update $R$ and ${Adj}$ by replacing $r_1, r_2$ with $r_{\text{new}}$
            \STATE Update $\mathcal{H}$ with new pairs involving $r_{\text{new}}$
        \ENDWHILE
        \STATE \textbf{return} $\mathcal{C}_{\text{f}}^{\text{c}} \leftarrow R$
        \end{algorithmic}
        \end{algorithm}
    \end{minipage}
\end{wrapfigure}
This eliminates the critical issue in overlapping methods, where a single point may belong to multiple regions, resulting in larger subproblems due to redundant coverage and making parallelization difficult, as subproblems are no longer independent.
The decomposition is performed in two steps: 
\begin{enumerate}[nosep]
    \item Triangulate $\mathcal C_{\mathrm f}$ in $O(|\mathcal V|\log|\mathcal V|)$ time, where $\mathcal V$ is a set of all vertices.
    \item Iteratively remove diagonals to merge adjacent triangles, but \emph{prioritize} edges that connect regions with the \emph{highest robot traffic} according to the MAPF solution $\bm{\Pi}_M$.  
\end{enumerate}
The resulting \textbf{Priority-Based Decomposition} keeps the number of regions proportional to where the robots actually need maneuvering, dramatically reducing the downstream dimension required for diffusion, as illustrated in Figure~\ref{fig:PBD}. The detailed procedure is shown in Algorithm~\ref{alg:Priority-based-Decomposition}.

\begin{theorem}[Sound, compact, and efficient]
\label{thm:pbd}
PBD outputs a finite set $\mathcal{C}_{\text{f}}^{\text{c}}$ of pairwise disjoint
convex polygons satisfying $\bigcup_{R\in \mathcal{C}_{\text{f}}^{\text{c}}}R \in \mathcal C_{\mathrm f}$. 
Its runtime is $O(|\mathcal V|\log|\mathcal V|)$, where
$\mathcal V$ are the vertices of the triangulation.
\end{theorem}
\begin{proof}[Sketch of Proof]
The runtime follows from (1) hole removal and triangulation in $\mathcal{O}(|\mathcal{V}|\log|\mathcal{V}|)$ time, and (2) at most $\mathcal{O}(|\mathcal{V}|)$ heap-based merge operations, each taking $\mathcal{O}(\log|\mathcal{V}|)$ time. 
\end{proof}

\begin{remark}[Region count]
Empirically, PBD eliminates up to 32.6\% of redundant convex regions across our benchmark scenarios. Compared to the baseline Visibility Clique Cover, it reduces the total number of region sets by over 50\%, while also achieving faster runtime. This directly translates to a comparable reduction in the number of diffusion model calls required for trajectory generation.
\end{remark}

\subsection{MAPF-Driven Spatiotemporal Assignment}
\label{sec:sta}

While PBD yields a non-overlapping spatial partition $\mathcal{C}_{\mathrm{f}}^{\mathrm{c}} = \{R_1, \dots, R_k\}$, this alone is insufficient for decomposing MRMP. In multi-robot systems, coordination hinges not only on spatial separation but also on the \emph{spatiotemporal ordering of robot actions}. Even when robots are assigned to disjoint regions, their trajectories can remain interdependent due to timing constraints. 
A central challenge lies in determining \emph{when} and \emph{where} each robot should enter or exit specific regions. These decisions introduce temporal coupling between otherwise spatially decoupled subproblems. Without resolving this temporal-spatial assignment, region-based planning cannot guarantee global feasibility.

To address this challenge, this paper makes a key observation: \emph{a discrete MAPF solution $\boldsymbol\Pi_M=\{\boldsymbol{\pi}_{m,i}\}_{i=1}^{N_a}$ inherently encodes a coarse spatiotemporal collision schedule}. 
Thus, rather than leveraging $\boldsymbol\Pi_M$ for its specific paths, we extract and exploit its embedded coordination structure. Although MAPF operates in a discretized space-time domain, it effectively resolves collisions and preserves essential inter-agent dependencies. 
In particular, for each robot $a_i$ and region $R_j$, we extract two key events:
\begin{itemize}[leftmargin=*, nosep]
    \item \textbf{Exit event:} Time $t_{\text{out}}$ and position $\boldsymbol{\pi}_{\text{out}}$, sich that $\boldsymbol{\pi}_{m,i}(t_{\text{out}}) \in R_j$ and $\boldsymbol{\pi}_{m,i}(t_{\text{out}}{+}1) \notin R_j$.
    \item \textbf{Entry event:} Time $t_{\text{in}}$ and position $\boldsymbol{\pi}_{\text{in}}$, such that $\boldsymbol{\pi}_{m,i}(t_{\text{in}}) \in R_k$ and $\boldsymbol{\pi}_{m,i}(t_{\text{in}}{-}1) \notin R_k$.
\end{itemize}
\smallskip 
From these, we construct a structured transition set: $\mathcal{T} = \big\{\, (a_i, R_j, t_{\mathrm{out}}, \boldsymbol{\pi}_{\mathrm{out}}, R_k, t_{\mathrm{in}}, \boldsymbol{\pi}_{\mathrm{in}}) \big\}$, where each tuple indicates that agent $a_i$ exits region $R_j$ at $(t_{\mathrm{out}}, \boldsymbol{\pi}_{\mathrm{out}})$ and enters region $R_k$ at $(t_{\mathrm{in}}, \boldsymbol{\pi}_{\mathrm{in}})$. These transitions provide start and goal constraints for region-level diffusion models, enabling temporally consistent and spatially decoupled planning.

The transition set $\mathcal{T}$ is constructed by a single forward pass through the MAPF trajectories and the region map (see Algorithm~\ref{alg:transition-extraction} for details), as illustrated in \textbf{S2} of Figure~\ref{fig:overview}.
Its structural utility is formalized as follows.
\begin{wrapfigure}[11]{r}{0.5\columnwidth} 
    \vspace{-\intextsep} 
    \begin{minipage}{\linewidth} 
        \begin{algorithm}[H] 
        \small
        \caption{Region Transition Extraction}
        \label{alg:transition-extraction}
        \begin{algorithmic}[1]
            \STATE \textbf{Input:} Convex regions $\mathcal{C}_{\text{f}}^{\text{c}}$, MAPF solution $\boldsymbol{\Pi}_{M}$
            \STATE \textbf{Output:} $\mathcal{T} = \{(a_i, R_j, t_{\text{out}}, \pi_{\text{out}}, R_k, t_{\text{in}}, \pi_{\text{in}})\}$
            \STATE $\mathcal{T} \leftarrow \emptyset$
            \FOR{robot $a_i$, region $R_j$ and $R_k$, time $t$}
                \IF{$\boldsymbol{\pi}_{M,i}(t) \in R_j$ and $\boldsymbol{\pi}_{M,i}(t+1) \in R_k$}
                    \STATE Record transition: $(a_i, R_j, t, \boldsymbol{\pi}_{M,i}(t), R_k, t+1, \boldsymbol{\pi}_{M,i}(t+1))$
                \ENDIF
            \ENDFOR
            \STATE \textbf{return} $\mathcal{T}$
        \end{algorithmic}
        \end{algorithm}
    \end{minipage}
\end{wrapfigure}
\begin{proposition}[Valid subproblem decomposition]
\label{prop:decouple}
Fix any transition set $\mathcal{T}$ obtained from $\boldsymbol{\Pi}_M$ and $\mathcal{C}_{\text{f}}^{\text{c}}$. Then:
\begin{enumerate}[nosep,label=(\alph*)]
  \item Each robot is uniquely assigned to one region $R \in \mathcal{C}_{\text{f}}^{\text{c}}$ at each time step.
  \item Within each region $R$, the set of robot trajectories is temporally bounded by entry/exit events in $\mathcal{T}$.
  \item There are no inter-region constraints at any time step, so each region defines an independent subproblem.
\end{enumerate}
\end{proposition}

\begin{proof}[Sketch of Proof]
Disjointness of regions ensures spatial uniqueness. The time-based conditions in Algorithm~\ref{alg:transition-extraction} ensure that each robot's regional assignment is temporally non-overlapping. Feasibility constraints in MRMP (e.g., inter-robot collision) only apply to robots simultaneously present in the same region. 
\qedhere
\end{proof}

\subsection{Diffusion-based trajectory generation}
By decomposing the original problem into multiple independent subproblems, we are able to operate within smaller, decoupled subspaces. This improves runtime efficiency \emph{and} sample quality. In this subsection, we describe how diffusion models are used to generate high-quality trajectories within each subproblem, as illustrated in {\bf S3} of Figure~\ref{fig:overview}.

\label{sec:diffusion}
\paragraph{Diffusion modesl -- preliminaries}
Denoising Diffusion Probabilistic Models~\cite{sohl2015deep, ho2020denoising, song2020score} define a generative process by learning to reverse a forward stochastic transformation that progressively corrupts structured data into noise. The generative model then approximates the inverse of this transformation to restore the original structure, allowing sampling from the learned distribution. 
The diffusion model is thus trained to minimize the difference between the estimated and true scores of the perturbed data. Once trained, the score network $\bm{s}_\theta$ is used to \emph{denoise} random samples \(\bm{x}_T \sim \mathcal{N}(0,I)\) by iteratively updating $\bm{x}_t$ along the \emph{score} direction until the samples resemble the original data. This is also known as the \emph{sampling phase}, which follows the Stochastic Gradient Langevin Dynamics (SGLD) update rule:
\begin{align}
    \bm{x}_t = \bm{x}_{t+1} + \frac{\epsilon}{2} s_\theta(\bm{x}_{t+1}, t+1) + \sqrt{\epsilon} \cdot z, \label{eq: sgld}
\end{align}
where $\epsilon$ denotes the step size and $z$ is standard Normal.


\paragraph{Trajectory Generation using Diffusion Models}
To facilitate the generation of collision-free trajectories by diffusion models, a straightforward method is to bias the sampling process by incorporating gradient-based guidance to encourage the robot to avoid obstacles and other robots. This can be achieved by adding a penalty term into Eq. (\ref{eq: sgld}): 
\[
\begin{aligned}
    \bm{x}_t = \bm{x}_{t+1} + \frac{\epsilon}{2} (s_\theta(\bm{x}_{t+1}, t+1)+ \mathcal{J}(\bm{x}_{t+1},\mathcal{O}))  + \sqrt{\epsilon} \cdot z,
\end{aligned}
\]
where
\[
\begin{aligned}
\mathcal{J}(\bm{x}_t,\mathcal{O}) =  \sum_{i=1}^{N_a} \nabla_{\bm{x}} d_o\left(\bm{x}_{t}^i, \mathcal{O}\right) + \sum_{i=1}^{N_a}\nabla_{\bm{x}} d_a\left(\bm{x}_{t}^i, \bm{x}_{t}^{-i}\right).
\end{aligned}
\]
We define the obstacle and inter-agent penalty functions based on the squared L2-norm distance:
\[
\begin{aligned}
d_o(\boldsymbol{\pi}_i, \mathcal{O}) &= \sum_{h=0}^H \sum_{o_j \in \mathcal{O}} \max\left\{ 0, r_{i,j}^o - \left\| \pi_i^h - o_j \right\|_2 \right\}, \\
d_a(\boldsymbol{\pi}_i, \boldsymbol{\Pi}_{-i}) &= \sum_{h=0}^H \sum_{\pi_j \in \boldsymbol{\Pi}_{-i}} \max\left\{ 0, r_{i,j}^a - \left\| \pi_i^h - \pi_j^h \right\|_2 \right\},
\end{aligned}
\]
where $d_o(\cdot)$ and $d_a(\cdot)$ measure the violation of minimum safety distances $r_{i,j}^o$ and $r_{i,j}^a$ between robot $a_i$ and obstacle $o_j$, and between robots $a_i$ and $a_j$, respectively. $\boldsymbol{\Pi}_{-i}$ denotes the set of trajectories of all robots except $a_i$.

Incorporating the penalty term into the sampling process allows the diffusion model to sample trajectories that follow the learned distribution while softly enforcing task requirements (e.g., collision avoidance). However, the guidance term alone does not guarantee that each trajectory remains within its assigned convex region, which is essential for maintaining independence across subproblems. 

Fortunately, since each subproblem is defined over a convex domain, we can efficiently enforce this constraint by projecting each intermediate sample back onto the feasible set. To this end, we integrate a convex projection operator
$\mathcal{P}_{\mathcal{C}_{\text{f}}^c}$ into the SGLD update rule with minimal computational overhead:
\[
\bm{x}_t = \mathcal{P}_{\mathcal{C}_{\text{f}}^c} \left( \bm{x}_t\right),
\]
where $\mathcal{P}_{\mathcal{C}_{\text{f}}^c}(\bm{x}) = \arg\min_{\bm{y} \in \mathcal{C}_{\text{f}}^c} \| \bm{x} - \bm{y} \|_2^2$ denotes the Euclidean projection onto the feasible convex region $\mathcal{C}_{\text{f}}^c$ associated with the specific subproblem.
\begin{theorem}[Obstacle avoidance]
\label{thm:Obstacle avoidance}
Let $\bm{x}_0$ denotes the trajectory generated by diffusion models with projection operator $\mathcal{P}_{\mathcal{C}_{\text{f}}^c}$. For arbitrary small $\xi$, there exist $t$ such that $\sum_{i=1}^{N_a} \nabla_{\bm{x}} d_o\left(\bm{x}_0, \mathcal{O}\right) \leq \xi$.
\end{theorem}

\begin{proof}[Sketch of Proof]
Since $\mathcal{C}_{\text{f}}^c$ is a convex set, the projection operator $\mathcal{P}_{\mathcal{C}_{\text{f}}^c}$ guarantees that the generated trajectory lies within this feasible set~\cite{christopher2024constrained}. Moreover, since $\bigcup_{R \in \mathcal{C}_{\text{f}}^{c}} R \subseteq \mathcal{C}_{\text{f}}$, and $\mathcal{C}_{\text{f}}$ denotes the free configuration space, the generated trajectories satisfy the obstacle avoidance constraints. 
\end{proof}

Additionally, beyond extracting timing structure, the MAPF plan $\boldsymbol{\Pi}_M$ provides a strong \textit{trajectory prior}. Thus, inspired by recent work on Diffusion Models Inversion~\cite{meng2022sdedit,rout2025semantic}, which shows that even an imperfect or suboptimal initial solution can lead to better performance than sampling from pure noise. Motivated by this, we adopt a similar strategy to improve sampling efficiency and quality. Unlike standard diffusion models that start from Gaussian noise $\bm{x}_T \sim \mathcal{N}(0, I)$, we leverage a valid but suboptimal solution from a MAPF solver to initialize the diffusion model's reverse process $\bm{x}_T = \boldsymbol{\Pi}_{M}$, which provides a strong structural prior for the sampling process.

\subsection{Constraint-aware Diffusion Refinement}
\label{sec:repair}
The previous stages ensure that each trajectory remains within its assigned convex region and satisfies obstacle avoidance constraints. For inter-robot collision avoidance, although multiple guidance methods have been introduced to encourage generating collision-free trajectories, these mechanisms do not provide strict guarantees, especially with a large number of robots. 

To address this limitation, we examine each region separately and introduce a \emph{Constraint-aware Diffusion} refinement step to repair the infeasible subproblems, as shown in \textbf{S4} of Figure~\ref{fig:overview}. If any constraint in~\eqref{eq:obst-avoid}–\eqref{eq:robot-avoid}
is violated, we run a constraints-aware diffusion model to obatin the feasible solution. Specifically, we introduce anther \emph{projection operator} $\mathcal{P}_{\boldsymbol{\Pi}}$ into the SGLD update rule:
\[
\bm{x}_t = \mathcal{P}_{\boldsymbol{\Pi}} \left( \bm{x}_t\right),
\]
where 
\[
\mathcal{P}_{\boldsymbol{\Pi}}(\bm{x}) = \arg\min_{\bm{y}} \| \bm{x} - \bm{y} \|_2^2 \quad \text{s.t.} \quad \text{Eq.~(\ref{eq:robot-avoid}), (\ref{eq:dyn})}.
\]

To solve $\mathcal{P}_{\boldsymbol{\Pi}}$ more efficiently, we use a Lagrangian dual method, where the constraints are incorporated into
a relaxed objective by using Lagrange multipliers $\nu$:
\begin{equation*}
\begin{aligned}
    \mathcal{L}(\bm{y}, \boldsymbol{\nu}) = & \| \bm{x} - \bm{y}\|_2^2 +  \boldsymbol{\nu} \sum\nolimits_{i=1}^{N_a}\nabla_{\bm{y}} d_a\left(\bm{y}_{t}^i, \bm{y}_{t}^{-i}\right),
\end{aligned}
\end{equation*}

To accelerate the convergence, a quadratic penalty term $\rho$ is introduced, and the augmented Lagrangian function is defined as:
\begin{equation*}
\begin{aligned}
    \mathcal{L}_{\text{alm}}(\bm{y}, \boldsymbol{\nu}) = \mathcal{L}(\bm{y}, \boldsymbol{\nu}) +  \rho \| \sum\nolimits_{i=1}^{N_a}\nabla_{\bm{y}} d_a\left(\bm{y}_{t}^i, \bm{y}_{t}^{-i}\right) \|^2.
\end{aligned}
\end{equation*}

A dual ascent strategy is used to solve the relaxed problem, where the primal variable $\boldsymbol{x}'$ is updated via gradient descent on the augmented Lagrangian $\mathcal{L}_{\text{alm}}(\bm{y}, \boldsymbol{\nu})$, and the dual variable $\boldsymbol{\nu}$ is updated according to $\boldsymbol{\nu} \leftarrow \boldsymbol{\nu} + \rho \left\| \sum\nolimits_{i=1}^{N_a} \nabla_{\bm{y}} d_a\left(\bm{y}_{t}^i, \bm{y}_{t}^{-i}\right) \right\|^2.$
The penalty coefficient $\rho$ is gradually increased to enforce the constraints more strictly. This iterative process continues until the residual of the collision avoidance constraints falls below a predefined threshold, at which point the algorithm returns a feasible trajectory.

\begin{remark}
Although inter-robot collision avoidance constraints are generally difficult to handle in MRMP due to their nonconvexity and high dimensionality, our approach focuses only on resolving such constraints within infeasible subproblems. These subproblems operate in significantly lower-dimensional spaces compared to the global problem. Moreover, by applying a Lagrangian relaxation, we further simplify the optimization landscape, making the constraint repair process highly efficient in practice.
\end{remark}

\section{Experiments}

\subsection{Experimental Setup}
\begin{figure}[t]
  \centering
  \begin{subfigure}[t]{0.2\columnwidth}
    \centering
    \includegraphics[width=\textwidth]{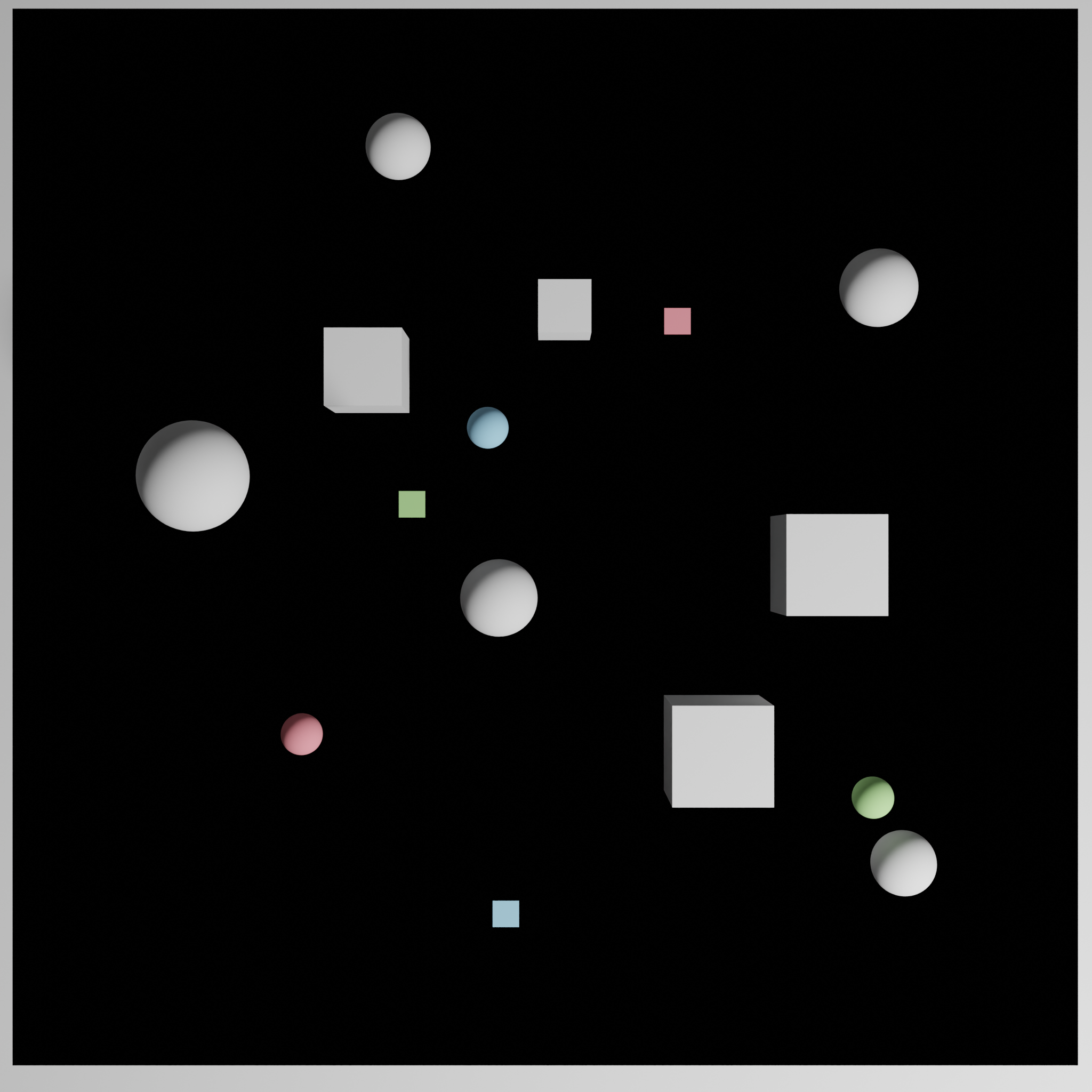}
    \caption{Basic.}
    \label{fig:Basic Maps}
  \end{subfigure}
  \hfill
  \begin{subfigure}[t]{0.2\columnwidth}
    \centering
    \includegraphics[width=\textwidth]{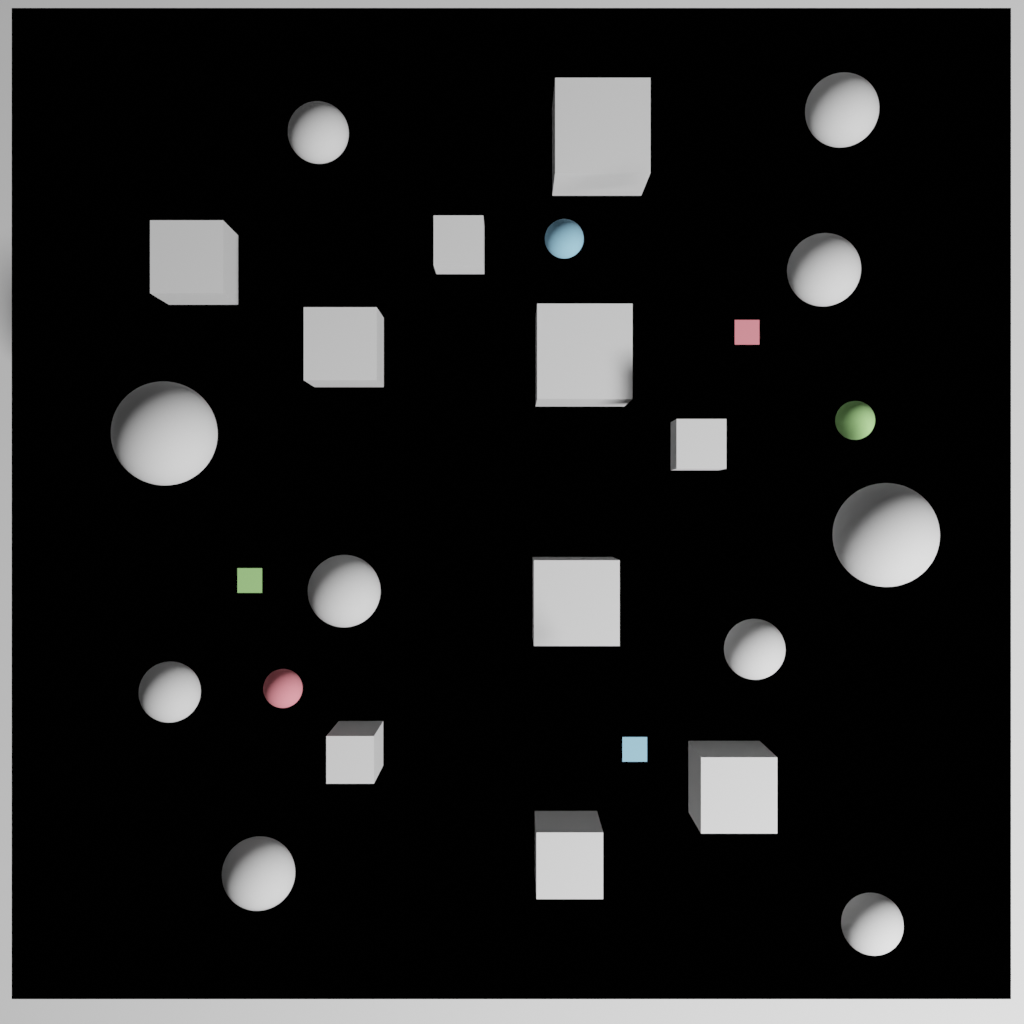}
    \caption{Dense.}
    \label{fig:Dense Maps}
  \end{subfigure}
  \hfill
  \begin{subfigure}[t]{0.2\columnwidth}
    \centering
    \includegraphics[width=\textwidth]{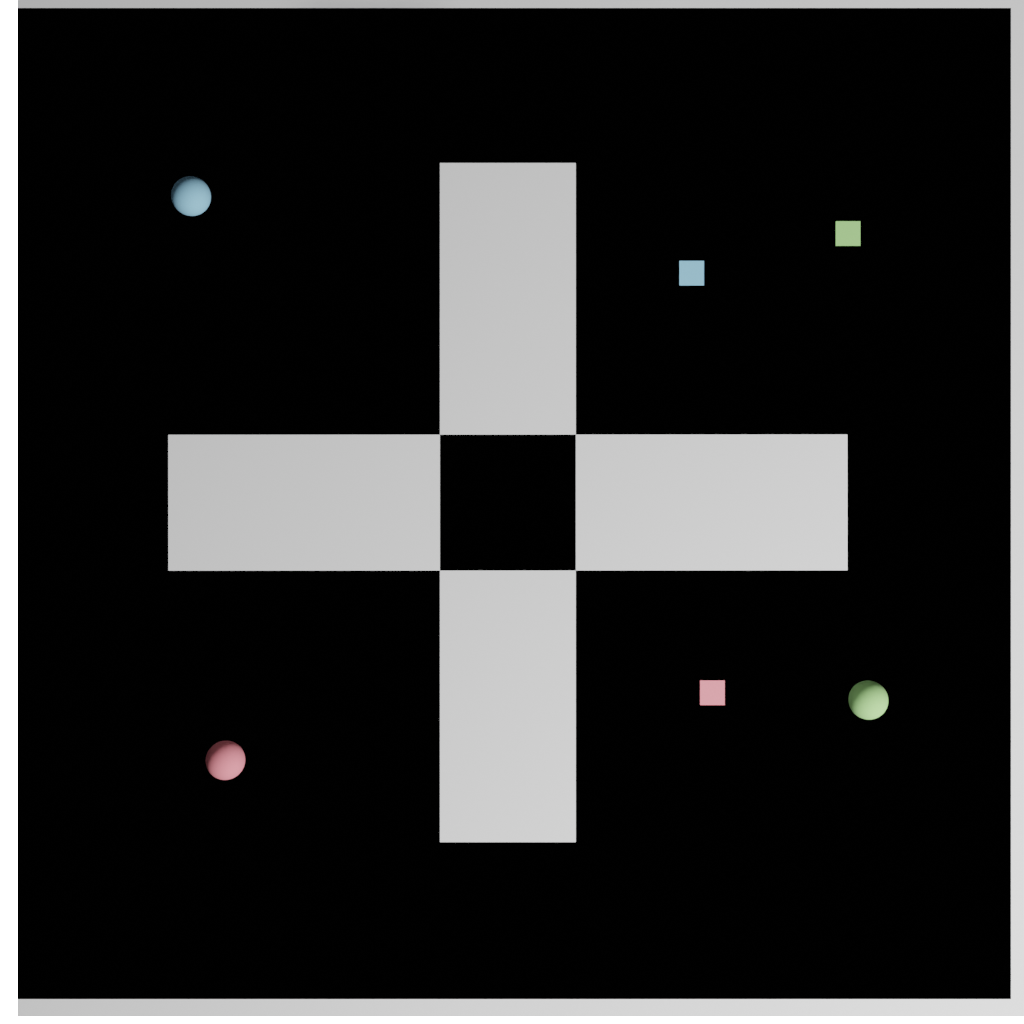}
    \caption{Room.}
    \label{fig:Room Maps}
  \end{subfigure}
  \hfill
  \begin{subfigure}[t]{0.2\columnwidth}
    \centering
    \includegraphics[width=\textwidth]{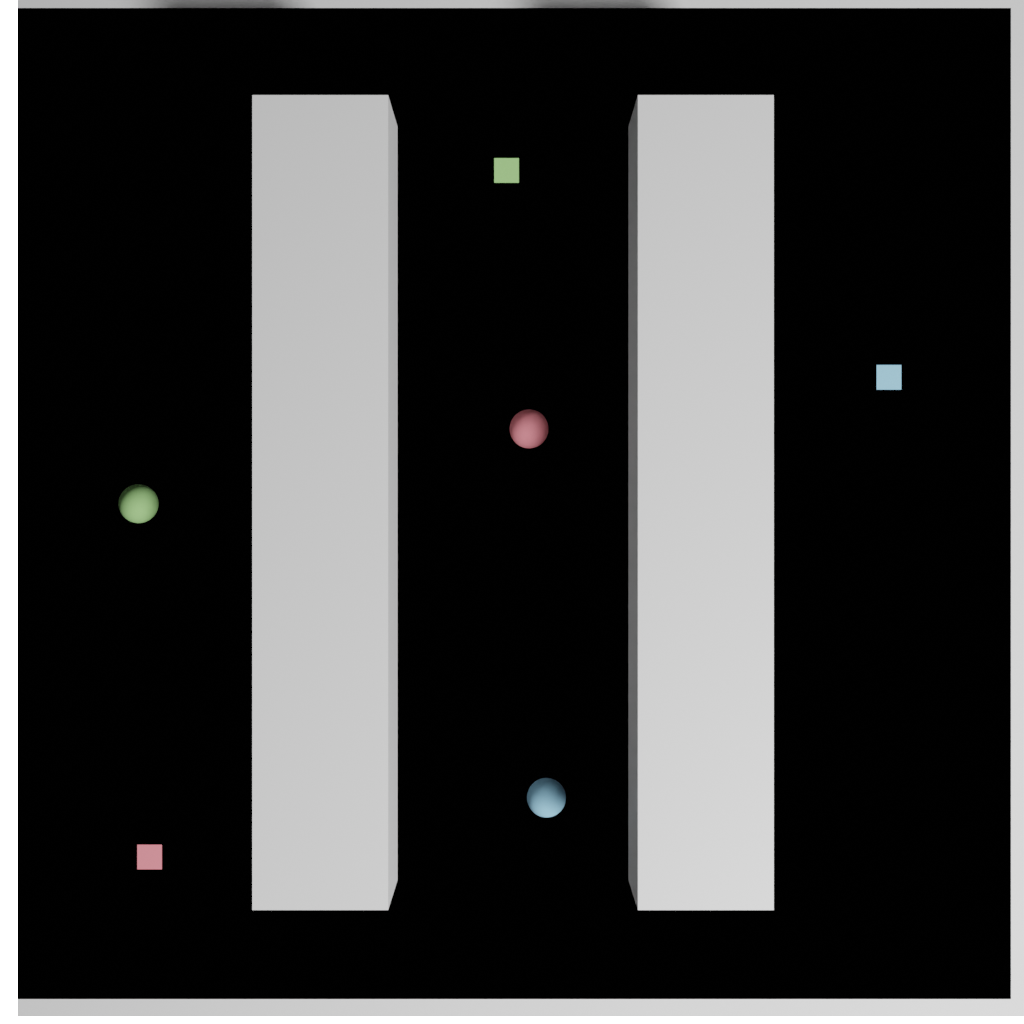}
    \caption{Shelf.}
    \label{fig:Shelf Maps}
  \end{subfigure}
  \caption{Examples of benchmark used for MRMP experiments, with increasing complexity. Colorful spheres and plates denote the start and goals of robots. White objects indicate obstacles.}
  \label{fig:Examples of benchmark}
\end{figure}

\paragraph{Settings.}
To evaluate the performance of diffusion-based MRMP algorithms, we design two experimental settings targeting different aspects of the problem.
\textbf{(1) Feasibility, Efficiency, Quality:} We adopt a diverse MRMP benchmark introduced in~\cite{liang2025simultaneous}, which includes four representative map types: random maps with varying obstacle densities (basic and dense), and structured maps mimicking real-world environments (room and shelf layouts). For each map type, we consider six robot counts (3, 6, 9, 12, 15, and 18), and evaluate across 25 instances with different configurations. Illustrations of the benchmark maps are included in the Figure~\ref{fig:Examples of benchmark}.
\textbf{(2) Scalability:} To evaluate the scalability of our method, we use a large-scale map containing over 100 heterogeneous obstacles. This setup supports scenarios with up to 100 robots, providing a challenging setting to examine the method’s ability to generate feasible plans in high-dimensional and cluttered environments.

\paragraph{Evaluation Metrics.}
Different algorithms are empirically evaluated based on \underline{Success rate}, \underline{Running time}, \underline{Path Length}, and \underline{Acceleration}. The \emph{Success rate} indicates the proportion of test cases solved without collisions and within the time limit (900 seconds), \emph{Running time} measures the computational efficiency required to generate a solution, \emph{Path Length} and \emph{Acceleration} evaluate the quality of generated trajectories.

\paragraph{Competing Methods.}
The proposed approach is compared against three SOTA diffusion-based baseline methods:
\begin{enumerate}[leftmargin=*, parsep=0pt, itemsep=0pt, topsep=0pt]
\item \textbf{Motion Planning Diffusion (MPD)}: A state-of-the-art diffusion model for single-robot motion planning~\cite{carvalho2023motion}, which we extend to the multi-robot setting for comparison.
\item \textbf{Multi-robot Motion Planning Diffusion (MMD)}: A method that integrates diffusion models with a classical search-based MAPF algorithm to generate MRMP solutions under soft collision constraints~\cite{shaoul2024multi}.
\item \textbf{Simultaneous MRMP Diffusion (SMD)}: The current SOTA method recently introduced in~\cite{liang2025simultaneous}. It integrates diffusion models with constrained optimization to generate collision-free trajectories for MRMP.
\end{enumerate}
The implementation details are provided in Appendix~\ref{app: Implementation Details}.

\subsection{Comparison across Methods}
\begin{table}[t]
\centering
\small
\setlength{\tabcolsep}{2pt}
\begin{tabular}{lc|>{\columncolor{lightgray}}c>{\columncolor{lightgray}}c>{\columncolor{lightgray}}c|ccc|ccc|ccc}
\toprule
Map & Robots 
& \multicolumn{3}{c|}{DGD} 
& \multicolumn{3}{c|}{MPD} 
& \multicolumn{3}{c|}{MMD} 
& \multicolumn{3}{c}{SMD} \\
\cmidrule(lr){3-5} \cmidrule(lr){6-8} \cmidrule(lr){9-11} \cmidrule(lr){12-14} 
& & S & T & P & S & T & P & S & T & P & S & T & P  \\
\midrule
\multirow{3}{*}{\rotatebox[origin=c]{90}{\textbf{Basic}}} 
& 6  & \textbf{100} & 12.2 & 1.21 & 76 & \textbf{9.1} & 1.12 & \textbf{100} & 27.1 & 1.12 & \textbf{100} & 254.3 & \textbf{1.10} \\
& 12 & 96 & 25.0 & 1.28 & 52 & \textbf{8.7} & \textbf{1.10} & \textbf{100} & 56.0 & 1.12 & N/A & N/A & N/A \\
& 18 & \textbf{96} & 65.7 & 1.33 & 8 & \textbf{9.4} & \textbf{1.09} & \textbf{96} & 86.3 & 1.13 & N/A & N/A & N/A \\
\midrule
\multirow{3}{*}{\rotatebox[origin=c]{90}{\textbf{Dense}}} 
& 6  & \textbf{100} & 12.3 & 1.26 & 4 & \textbf{9.2} & \textbf{1.11} & 40 & 36.5 & 1.15 & \textbf{100} & 287.3 & 1.13 \\
& 12 & \textbf{100} & \textbf{32.9} & 1.35 & 0 & N/A & N/A & 8 & 62.5 & \textbf{1.15} & N/A & N/A & N/A \\
& 18 & \textbf{76} & \textbf{47.8} & 1.37 & 0 & N/A & N/A & 8 & 87.1 & \textbf{1.18} & N/A & N/A & N/A \\
\midrule
\multirow{3}{*}{\rotatebox[origin=c]{90}{\textbf{Room}}} 
& 6  & \textbf{100} & \textbf{38.9} & 1.44 & 0 & N/A & N/A & 24 & 55.5 & \textbf{1.16} & \textbf{100} & 181.0 & 1.19 \\
& 12 & \textbf{100} & \textbf{84.0} & 1.50 & 0 & N/A & N/A & 4 & 122.6 & \textbf{1.14} & N/A & N/A & N/A \\
& 18 & \textbf{100} & \textbf{219.0} & \textbf{1.55} & 0 & N/A & N/A & 0 & N/A & N/A & N/A & N/A & N/A \\
\midrule
\multirow{3}{*}{\rotatebox[origin=c]{90}{\textbf{Shelf}}} 
& 6  & \textbf{100} & \textbf{64.6} & 1.30 & 0 & N/A & N/A & 32 & 48.5 & 1.18 & \textbf{100} & 274.3 & \textbf{1.16} \\
& 12 & \textbf{100} & \textbf{131.5} & 1.33 & 0 & N/A & N/A & 4 & 95.1 & \textbf{1.22} & N/A & N/A & N/A \\
& 18 & \textbf{84} & \textbf{324.8} & \textbf{1.38} & 0 & N/A & N/A & 0 & N/A & N/A & N/A & N/A & N/A \\
\bottomrule
\end{tabular}
\caption{Success rate (S, in percentage), running time (T, in seconds), and path length (P).}
\label{tab:mrmp_benchmark}

\end{table}

We evaluate the performance of all methods on the standard MRMP benchmark, comparing success rates, runtimes, and path lengths across scenarios of increasing complexity and robot count (6, 12, and 18 agents). Full results are summarized in Table~\ref{tab:mrmp_benchmark}, with additional metrics and ablations provided in Appendix~\ref{app: Additional Results}.

\subsection*{Baseline Performance}

\noindent {\bf Motion Planning Diffusion (MPD)} achieves the lowest runtimes and path lengths when successful. This is the case for the simplest problems (Basic maps with 6 robots). However, MPD fails to scale, consistently failing to generate feasible solutions beyond three robots in any more complex scenario. This underscores the limitations of unguided diffusion in multi-agent settings with tight coupling.

\noindent{\bf Multi-robot Motion Planning Diffusion (MMD)} improves substantially over MPD by jointly modeling all agents. It maintains high success rates on Basic maps with up to 18 robots. However, its performance sharply degrades in more constrained environments. For instance, on Room and Shelf maps with 12 or more agents, MMD fails consistently. This breakdown is attributed to its lack of explicit coordination mechanisms, which leads to congestion and deadlocks in narrow passages.

\noindent{\bf Simultaneous MRMP Diffusion (SMD)}  shows high-quality planning with 100\% success rates for up to 6 robots. However, its computational cost is orders of magnitude higher than other methods. For 12 or more robots, the running time becomes prohibitively long, and no solution is obtained within the imposed time limit.

\subsection*{Discrete Guided Diffusion (DGD)}
DGD outperforms all baselines in both success rate and computational efficiency. By decomposing the MRMP problem into multiple region-level subproblems using a structured MAPF schedule, DGD solves each subproblem independently and in parallel, significantly reducing planning overhead.
DGD \emph{attains over 92\% success rate across all settings} except the most challenging (Dense and Shelf maps with 18 robots), where it still \emph{leads with a success rate of 76\%}.
Additionally, it achieves such results with lowest runtimes for the non-trivial benchmarks. 
Notably, on Dense maps with 6 robots, DGD matches SMD’s perfect success rate while using only 4\% of its runtime.

DGD’s path lengths are comparable to those of MPD and MMD. While the average path length appears slightly higher, this is primarily because the averages are computed only over successful trials. In more challenging scenarios, where MPD and MMD often fail, DGD still succeeds, albeit with longer paths due to the complexity of the environment. Example trajectories generated by DGD are shown in Figure~\ref{fig:Generated Trajectories on MRMP Benchmark.} and those for the baselines in Appendix~\ref{app: Additional Results}.

\begin{figure}[t]
  \centering
  \begin{subfigure}[t]{0.4\columnwidth}
    \centering
    \includegraphics[width=\textwidth]{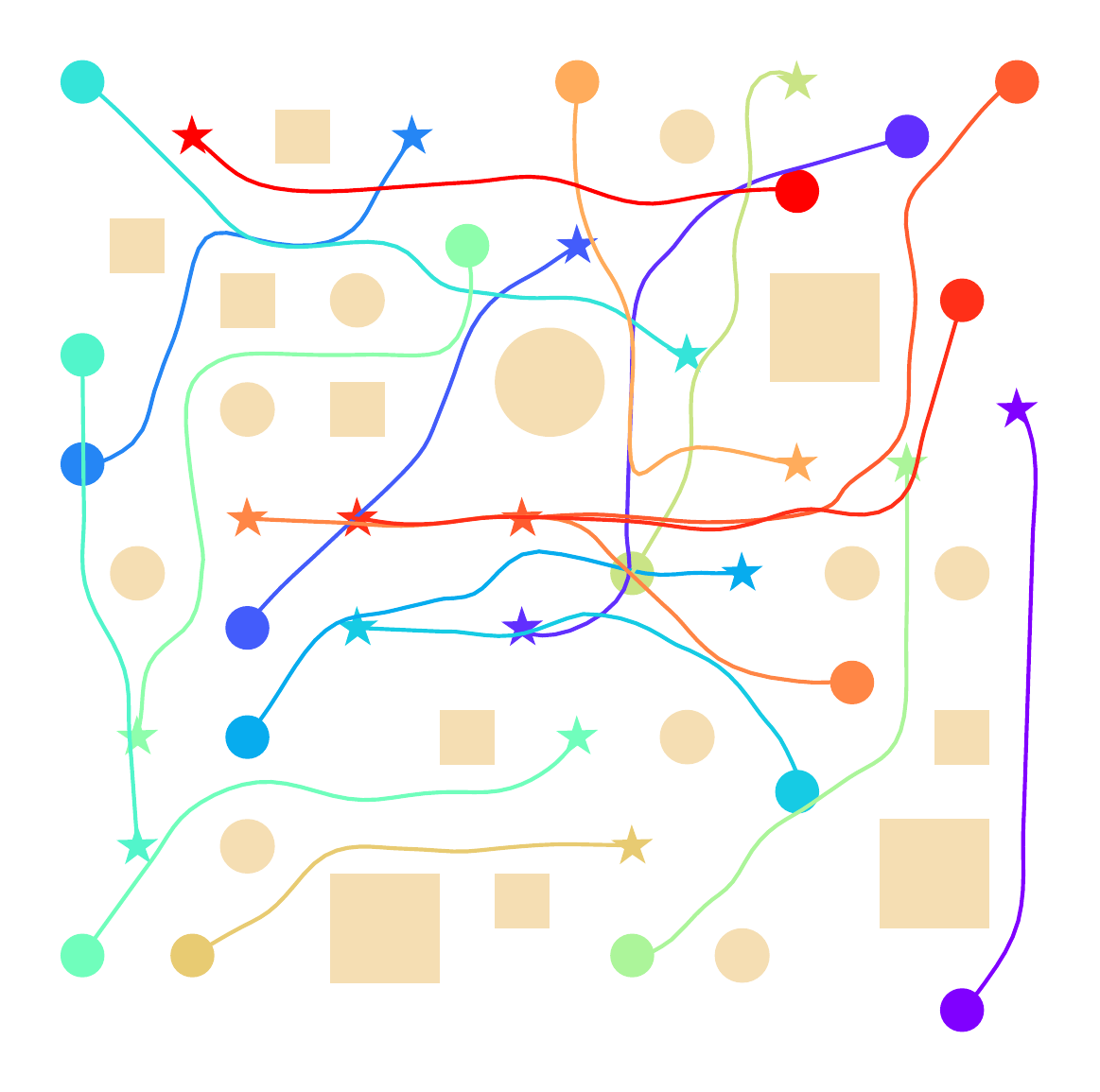}
    \caption{Dense Maps.}
    \label{fig:Trajectories for Dense Maps}
  \end{subfigure}
  \hspace{16pt}
  \begin{subfigure}[t]{0.4\columnwidth}
    \centering
    \includegraphics[width=\textwidth]{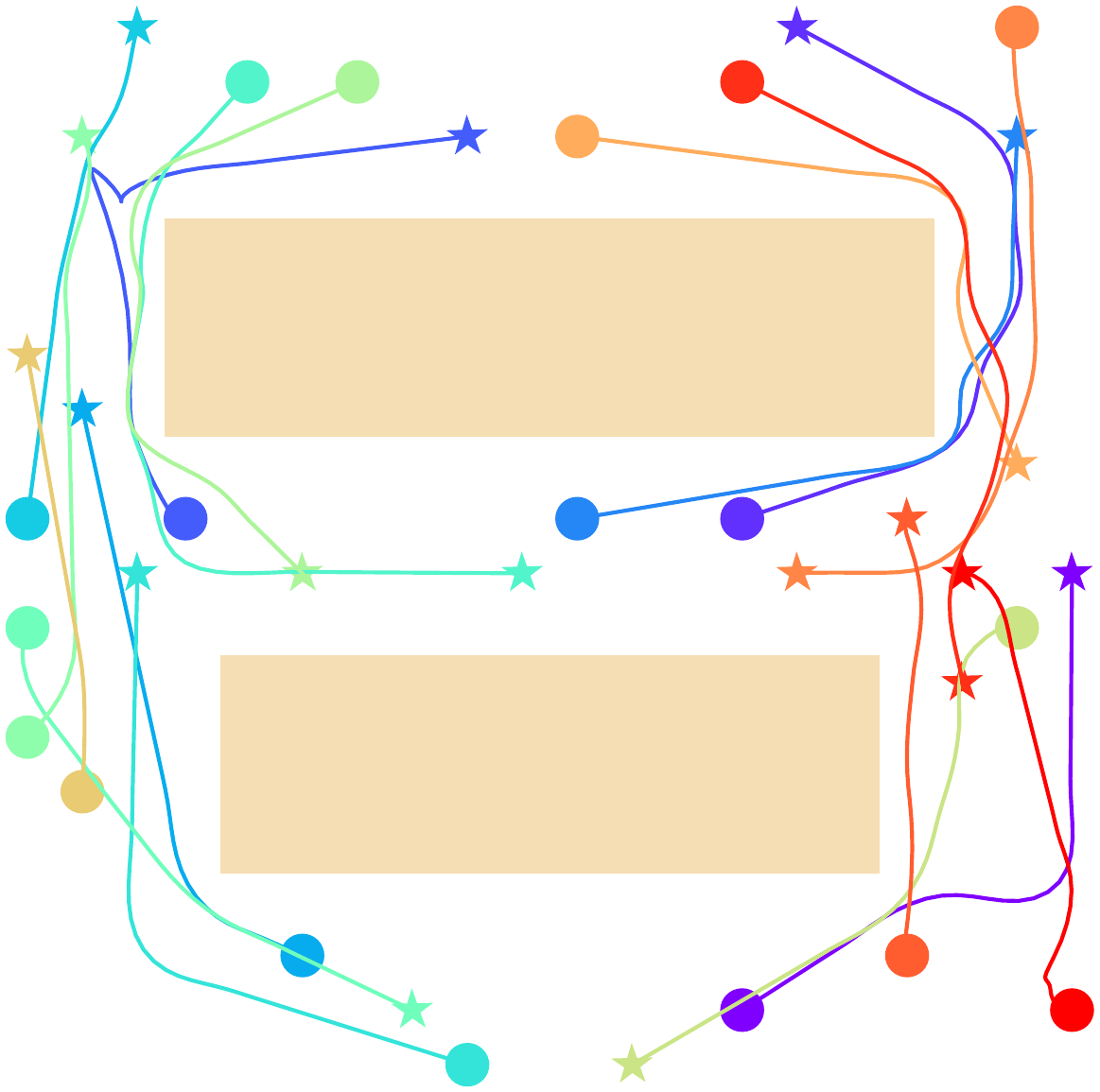}
    \caption{Shelf Maps.}
    \label{fig:Trajectories for Shelf Maps}
  \end{subfigure}
  \caption{Trajectories generated by DGD.}
  \label{fig:Generated Trajectories on MRMP Benchmark.}
\end{figure}

\begin{figure}[t]
    \centering
    \begin{tikzpicture}
        \node[inner sep=0pt] (img) at (0,0) {\includegraphics[width=0.85\columnwidth]{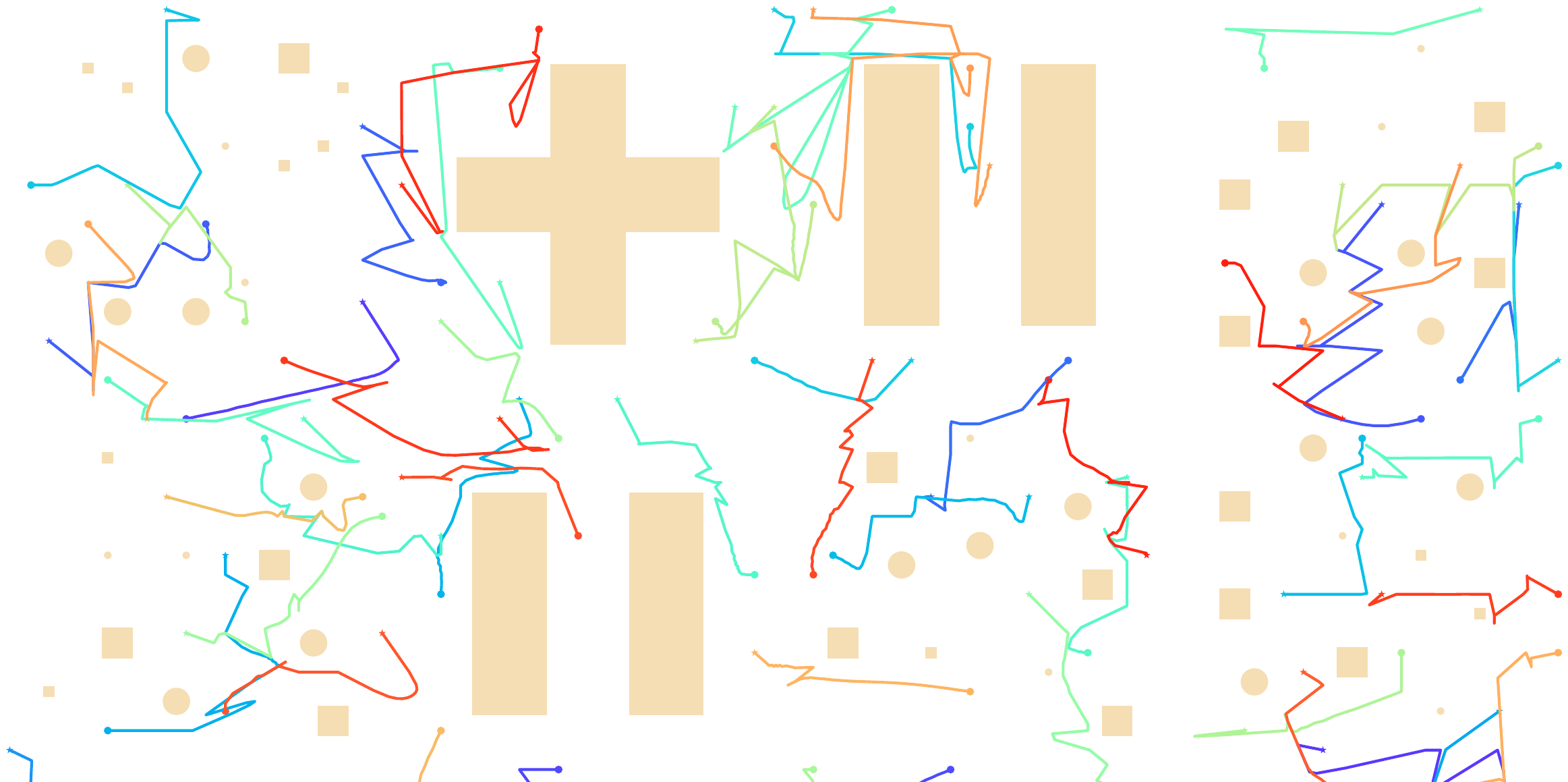}};
    \draw[dashed, gray, thick] 
        ([yshift=-3.5cm] img.center) ++(-7cm,0) -- ++(14cm,0);
    \end{tikzpicture}
    \caption{Trajectories generated by DGD on large maps. The figure is horizontally cropped for readability, with a dashed line on the bottom indicating continuation. See Appendix~\ref{app: Additional Results} for the full version.}
    \label{fig:generated_trajectories_large}
\end{figure}

\subsection{Scalability Analysis}
To evaluate the scalability of DGD, we test it on large-scale environments containing a diverse mix of obstacle types within a single map. These environments feature {\bf 104} obstacles and {\bf 100} robots, substantially exceeding the complexity and scale of standard MRMP benchmarks.

As in earlier experiments, DGD decomposes the global MRMP problem into a set of independent subproblems, each confined to a convex region. This decomposition enables efficient parallelized planning, and allows DGD to scale gracefully to high-dimensional instances. Representative trajectories generated by DGD in this large-scale setting are shown in Figure~\ref{fig:generated_trajectories_large} (figure cropped due to large size).

To the best of our knowledge, DGD is the first diffusion-based MRMP method capable of generating feasible solutions in environments with over 100 obstacles and 100 robots. In contrast to prior diffusion approaches, which have been demonstrated only on up to 40 robots in obstacle-free settings~\cite{shaoul2024multi}, \emph{DGD achieves a 2.5$\times$ improvement in robot count} while simultaneously addressing \emph{significantly more complex environments}. 
This is significant: scalability is a key impediment to apply generative models to real-world multi-robot systems, where both geometric complexity and agent coordination pose significant challenges, and 
Thus, this works makes a significant step towards practical diffusion-based solvers for multi-robot planning.

\subsection{Comparison between DGD and MAPF}
We present a comparison between the MAPF solution and the trajectory generated by our DGD method for the same instance. As shown in Figure~\ref{fig:Trajectories generated by DGD and MAPF.}, the MAPF solution is restricted to a discrete grid, resulting in suboptimal paths with unnecessary turns and increased travel distance. In contrast, DGD generates trajectories with shorter path lengths while maintaining collision-free guarantees. For example, the highlighted blue robot follows a near-optimal straight-line path under DGD, whereas the corresponding MAPF trajectory is longer and more convoluted. This comparison illustrates that DGD not only benefits from the spatiotemporal guidance provided by the MAPF solution but also achieves higher-quality trajectories by operating in continuous space.
\begin{figure}[t]
  \centering
  \begin{subfigure}[t]{0.4\columnwidth}
    \centering
    \includegraphics[width=\textwidth]{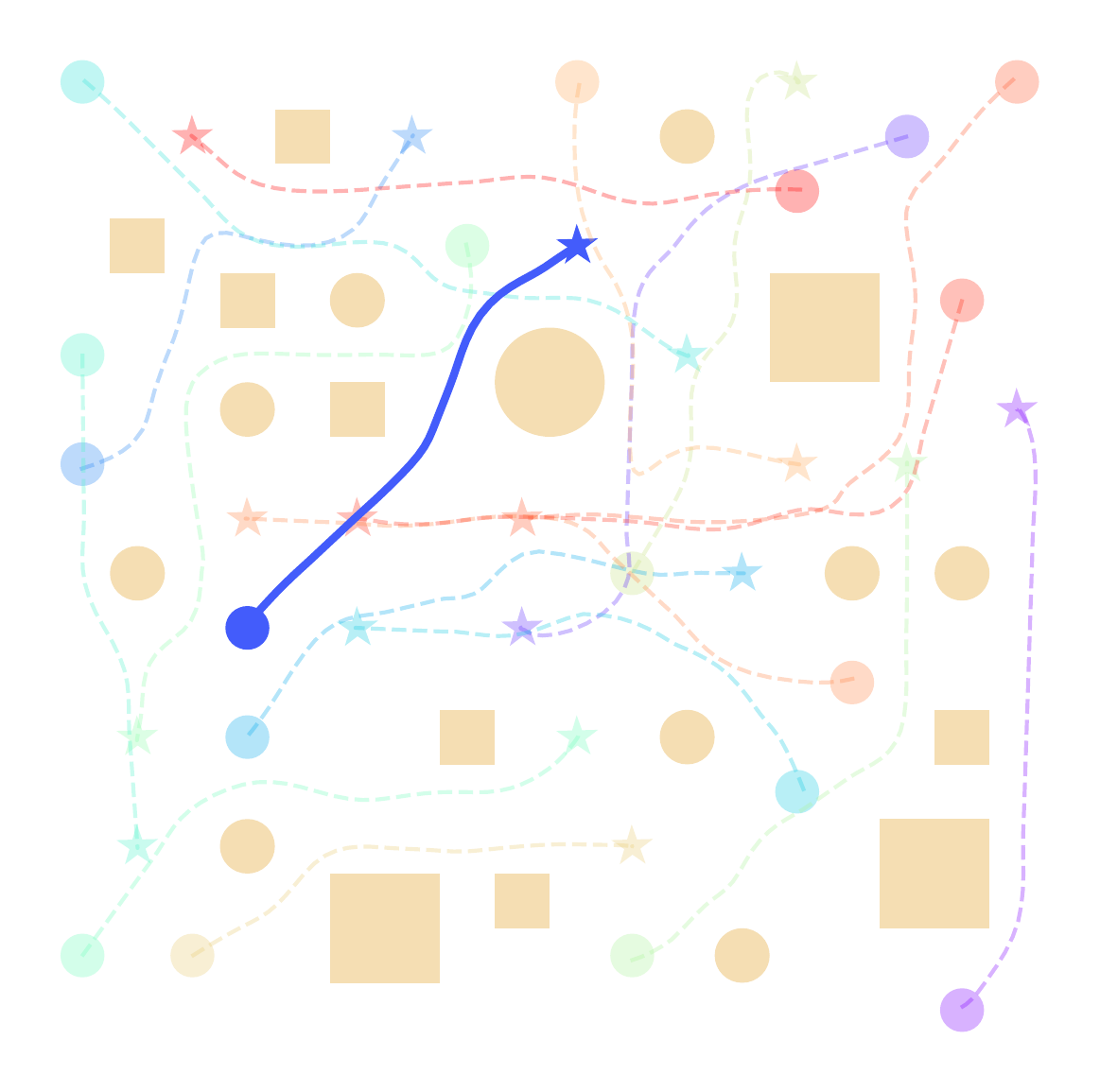}
    \caption{DGD.}
    \label{fig: DGD traj}
  \end{subfigure}
  \hspace{16pt}
  \begin{subfigure}[t]{0.4\columnwidth}
    \centering
    \includegraphics[width=\textwidth]{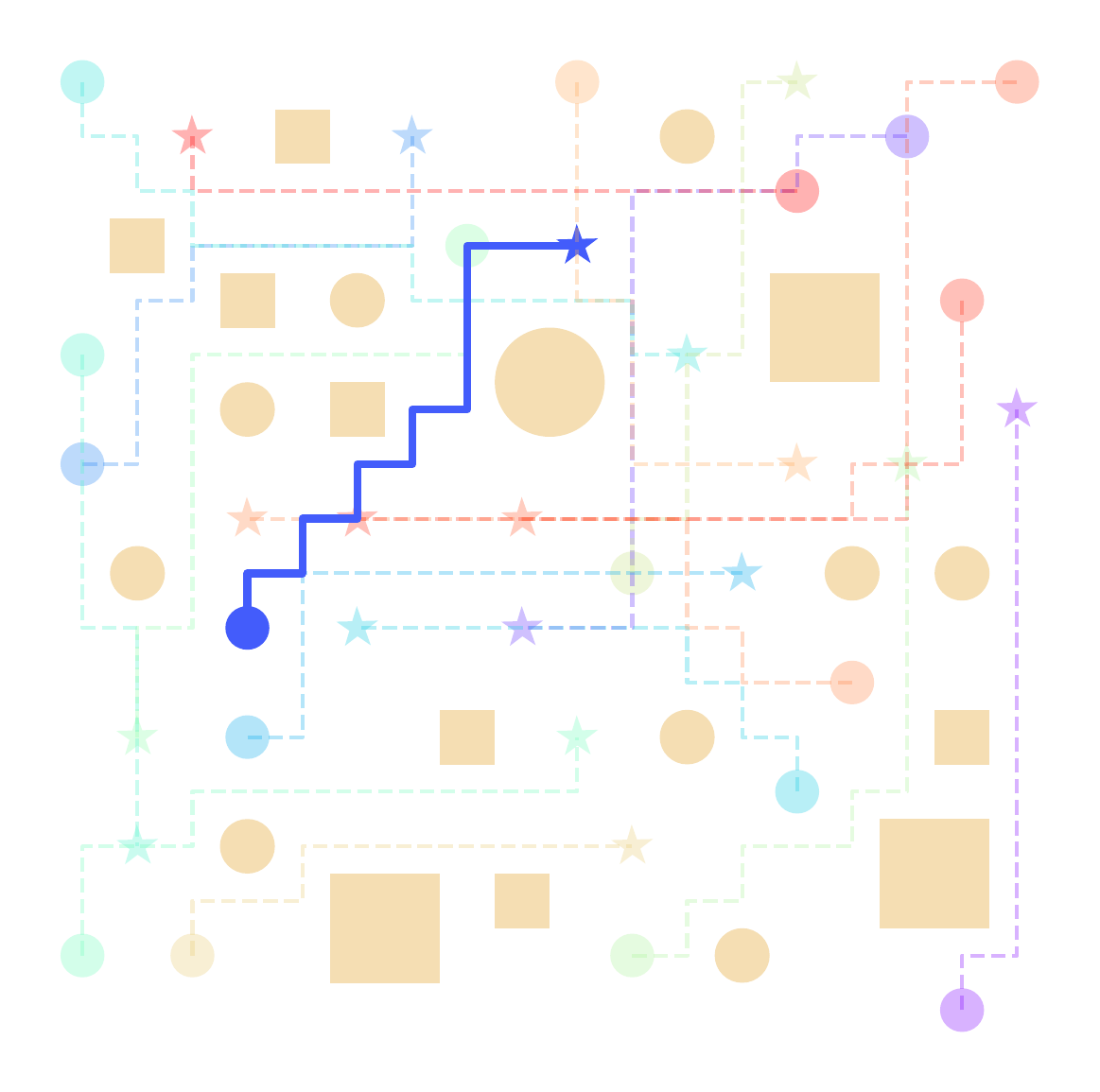}
    \caption{MAPF.}
    \label{fig: MAPF trak}
  \end{subfigure}
  \caption{Trajectories generated by DGD and MAPF.}
  \label{fig:Trajectories generated by DGD and MAPF.}
\end{figure}

\subsection{Limitation Anslysis}
\begin{wrapfigure}[13]{r}{0.3\columnwidth} 
\vspace{-\intextsep}
    \centering
    \begin{tikzpicture}
        \node[inner sep=0pt] (img) at (0,0) {\includegraphics[width=\linewidth]{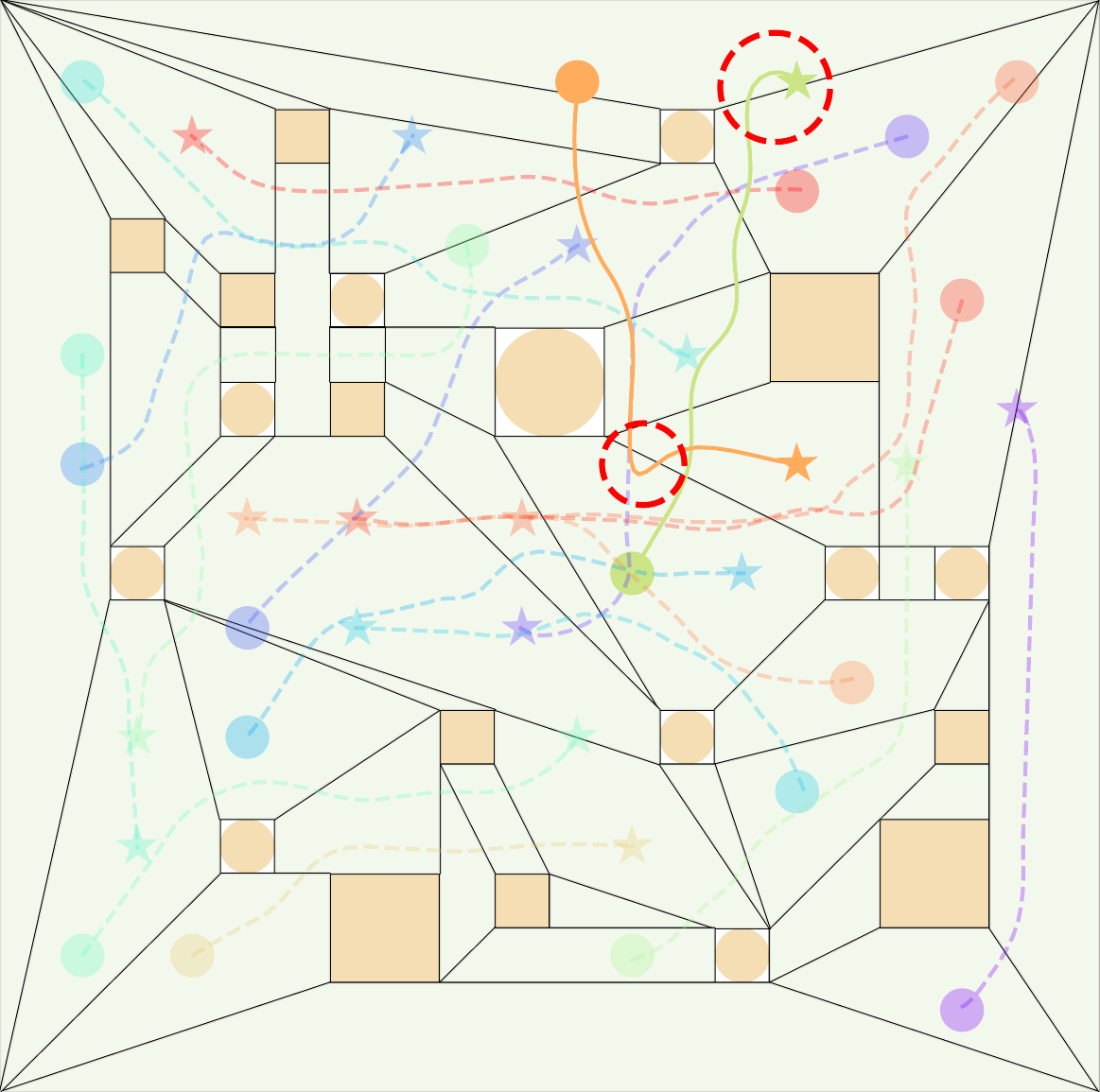}}; 
        \draw[dashed, gray, thick] 
            ([xshift=-2.5cm] img.center) ++(0,-2.5cm) -- ++(0,5cm);
        \draw[dashed, gray, thick] 
            ([yshift=-2.5cm] img.center) ++(-2.5cm,0) -- ++(5cm,0);
    \end{tikzpicture}
    \caption{Partial trajectories.}
    \label{fig:Trajectories generated by DGD on a dense map with 18 robots.}
\end{wrapfigure}
We believe that DGD offers strong potential for advancing multi-robot coordination. One limitation, however, lies in the relatively high acceleration occasionally observed in DGD-generated trajectories, which is caused by abrupt changes in velocity. Figure~\ref{fig:Trajectories generated by DGD on a dense map with 18 robots.} shows a subset of trajectories generated by DGD on a dense map with 18 robots. The red dashed circles highlight two regions where robots undergo sharp velocity changes, corresponding to the green and orange robots. Gray dashed lines indicate the cropped boundary.

These discontinuities typically arise near the boundaries between adjacent convex regions and result from the absence of explicit consistency constraints across subproblems during trajectory generation. While DGD effectively generates collision-free trajectories by leveraging spatiotemporal priors, enforcing inter-region consistency is a promising direction for improving trajectory smoothness and physical realism. Despite such dynamics, DGD ensures collision-free paths by leveraging continuous-space generation guided by spatiotemporal priors. 

\section{Conclusion}
This work introduces \emph{Discrete Guided Diffusion (DGD)}, a method that leverages a decomposition strategy to enable generative diffusion models to efficiently generate collision-free trajectories for multi-robot systems. By decomposing the original MRMP problem into multiple tractable subproblems, DGD alleviates the significant computational burden associated with large-scale MRMP. To address the spatiotemporal dependencies across multiple robots, MAPF solutions are employed to guide diffusion models in generating solutions for each subproblem. In addition, an efficient constraint repair mechanism ensures the feasibility of the generated trajectories.

Extensive experiments across varying obstacle densities and increasing robot counts demonstrate that DGD achieves significantly higher success rates and lower running times than competing methods. Experiments on large-scale maps further validate its scalability and robustness in handling both static obstacles and dynamic interactions, making it a promising approach for real-world applications.

While DGD achieves state-of-the-art performance on MRMP tasks, there are several promising directions to further enhance its capabilities. One opportunity lies in improving the decomposition process to produce more flexible subproblem partitions. In addition, extending the current framework to incorporate interactions between adjacent subregions could help better preserve temporal and spatial consistency across the full trajectory.

\section*{Acknowledgment}
This research is partially supported by NSF grants 2334936, 2334448, 2533631, and NSF CAREER Award 2401285.  
The authors acknowledge Research Computing at the University of Virginia for providing computational resources that have contributed to the results reported within this paper. 
The views and conclusions of this work are those of the authors only.
\bibliographystyle{plainnat} 
\bibliography{references}

\clearpage
\appendix

\section{Missing Proof}
\label{app: missing_proof}
\subsection{Proof of Theorem~\ref{thm:pbd}}
\label{app:proof_pbd}
\begin{proof}

We analyze the time complexity of the algorithm by examining its two main phases: Triangulation and Iterative Merging. Let \(|\mathcal V|\) denote the number of vertices. 

\paragraph{(i) Triangulation.}  
The main idea of triangulation is first to decompose the original nonconvex configuration space $C_{\text{f}}$ into simple polygons and then triangulate these simple polygons. 

According to Theorem 3.6 in~\cite{de2008computational}, the decomposition of an arbitrary simple polygon with \(|\mathcal V|\) vertices into a set of y-monotone components can be executed in $O(|\mathcal V|\log|\mathcal V|)$ time. Once this decomposition is achieved, Theorem 3.7 in~\cite{de2008computational} establishes that each strictly y-monotone polygon admits a triangulation in linear time with respect to the number of its vertices. 

Combining both stages, the overall computational complexity of triangulating $C_{\text{f}}$ is thus asymptotically dominated by the monotone decomposition step, yielding a total time complexity of $O(|\mathcal V|\log|\mathcal V|)$.

\paragraph{(ii) Iterative Merging.}  
Following triangulation, the algorithm enters a greedy, priority-guided region merging phase to produce a compact convex decomposition. Let $O(|\mathcal V|)$ denote the number of initial triangles. These serve as the base regions for merging.

The merging process is implemented via a max-heap data structure $\mathcal H$, which stores all admissible region pairs $(r_1, r_2)$ eligible for merging. Each element in $\mathcal H$ is keyed by a priority function $p(r)$ that encodes topological information. For each region $r \in R$, adjacency information is precomputed, leading to $O(|\mathcal V|)$ total adjacency relations, each inserted once into the heap during initialization.

Each successful merge operation reduces the total number of regions by exactly one. Thus, in the worst case, there are at most $O(|\mathcal V|)$ merges. Each merge operation requires a constant number of heap extractions and updates, each incurring a logarithmic overhead due to the heap structure. Therefore, the total cost of all merge operations is bounded above by $ $.

\subparagraph{Overall Complexity..} 
Combining both phases, the algorithm exhibits a worst-case runtime of $O(|\mathcal V|\log|\mathcal V|)$. 
\end{proof}

\subsection{Proof of Proposition~\ref{prop:decouple}}
\begin{proof}
Because $\{R_j\}$ is a partition of free space into disjoint convex regions, each robot position $\boldsymbol\pi_{m,i}(t)$ lies in exactly one region, so (a) holds.  

Algorithm \ref{alg:transition-extraction} records the first and last timesteps at which each robot enters and exits a region, which bounds its occupancy interval in that region and yields (b).  

Finally, collision and feasibility constraints only involve robots sharing the same region at the same time; since no two robots coincide outside their recorded intervals, regions induce independent subproblems, giving (c).
\end{proof}

\subsection{Proof of Theorem~\ref{thm:Obstacle avoidance}}
\begin{proof}
Beacause he projection operator $\mathcal{P}_{C_{\text{f}}^c}(\bm{x})$ performs the following convex optimization problem at each iteration:
\begin{align}
    \bm{x} = \ & \arg \min_{\bm{y} \in C_{\text{f}}^c} \| \bm{x} - \bm{y} \|_2^2\ .
\end{align}
Since $C_{\text{f}}^c$ is convex and $\| \bm{x} - \bm{y} \|_2^2$ is convex and continuously differentiable over $C_{\text{f}}^c$, the optimization problem has a unique global minimum.

Let $Error$ be the distance between $\boldsymbol{x}_{t}$ and its nearest feasible point. Using Corollary 3 in~\cite{christopher2025neuro}, for any arbitrarily small $\xi > 0$, there exists a time $t$ such that after the update:
\begin{align}
    \label{eq: feasibility_guarantee}
\mathbb{E} \left[ \textit{Error}(\mathcal{U}(\mathcal{P}_{C_{\text{f}}^c}(\bm{x})), C_{\text{f}}^c) \right] \leq \xi .  
\end{align}
where $\mathcal{U}(\cdot)$ is the update operator defined by Langevin dynamics.

Therefore, the expected distance between the final generated trajectory and the feasible convex regions $C_{\text{f}}^c$ is bounded above by $\xi$, which means our DGD method yields a strictly smaller expected constraint violation and provides a feasibility guarantee for the convex constraint set $C_{\text{f}}^c$.

\end{proof}

\section{Additional Results}
\label{app: Additional Results}

\subsection{Detailed Comparison across Methods}
We provide detailed results for all methods across four metrics: success rate, running time, path length, and acceleration. These results are presented in Table~\ref{tab: Success rate (S, in percentage) and running time (T, in seconds) across all methods.} and Table~\ref{tab: Path length and acceleration across all methods.}.

\begin{table}[t]
\centering
\small
\setlength{\tabcolsep}{2pt}
\begin{tabular}{lc|cc|cc|cc|cc|cc}
\toprule
Map & Robots 
& \multicolumn{2}{c|}{DGD} 
& \multicolumn{2}{c|}{DM} 
& \multicolumn{2}{c|}{MPD} 
& \multicolumn{2}{c|}{MMD} 
& \multicolumn{2}{c}{SMD} \\
\cmidrule(lr){3-4} \cmidrule(lr){5-6} \cmidrule(lr){7-8} \cmidrule(lr){9-10} \cmidrule(lr){11-12}
& & S & T & S & T & S & T & S & T & S & T \\
\midrule
\multirow{6}{*}{\rotatebox[origin=c]{90}{\textbf{Basic}}} 
& 3  & \textbf{100} & 6.5 & 12 & \textbf{4.7} & 92 & 9.3 & \textbf{100} & 13.1 & \textbf{100} & 76.6 \\
& 6  & \textbf{100} & 12.2 & 0 & N/A & 76 & \textbf{9.1} & \textbf{100} & 27.1 & \textbf{100} & 254.3 \\
& 9  & \textbf{100} & 17.4 & 0 & N/A & 72 & \textbf{9.9} & \textbf{100} & 41.4 & \textbf{100} & 504.0 \\
& 12 & 96 & 25.0 & 0 & N/A & 52 & \textbf{8.7} & \textbf{100} & 56.0 & N/A & N/A \\
& 15 & 96 & 41.4 & 0 & N/A & 12 & \textbf{10.3} & \textbf{100} & 70.9 & N/A & N/A \\
& 18 & \textbf{96} & 65.7 & 0 & N/A & 8 & \textbf{9.4} & \textbf{96} & 86.3 & N/A & N/A \\
\midrule
\multirow{6}{*}{\rotatebox[origin=c]{90}{\textbf{Dense}}} 
& 3  & \textbf{100} & 6.6  & 0 & N/A & 88 & \textbf{8.3} & 60 & 19.1 & \textbf{100} & 81.2 \\
& 6  & \textbf{100} & 12.3 & 0 & N/A & 4 & \textbf{9.2} & 40 & 36.5 & \textbf{100} & 287.3 \\
& 9  & \textbf{100} & \textbf{18.3} & 0 & N/A & 0 & N/A & 28 & 51.4 & \textbf{100} & 582.1 \\
& 12 & \textbf{100} & \textbf{32.9} & 0 & N/A & 0 & N/A & 8 & 62.5 & N/A & N/A \\
& 15 & \textbf{96} & \textbf{37.0} & 0 & N/A & 0 & N/A & 4 & 72.4 & N/A & N/A \\
& 18 & \textbf{76} & \textbf{47.8} & 0 & N/A & 0 & N/A & 8 & 87.1 & N/A & N/A \\
\midrule
\multirow{6}{*}{\rotatebox[origin=c]{90}{\textbf{Room}}} 
& 3  & \textbf{100} & 17.1  & 4 & \textbf{4.1} & 12 & 8.8 & 60 & 28.2 & \textbf{100} & 72.6 \\
& 6  & \textbf{100} & \textbf{38.9} & 0 & N/A & 0 & N/A & 24 & 55.5 & \textbf{100} & 181.0 \\
& 9  & \textbf{100} & \textbf{59.9} & 0 & N/A & 0 & N/A & 8 & 142.4 & 96 & 239.4 \\
& 12 & \textbf{100} & \textbf{84.0} & 0 & N/A & 0 & N/A & 4 & 122.6 & N/A & N/A \\
& 15 & \textbf{100} & \textbf{115.5} & 0 & N/A & 0 & N/A & 0 & N/A & N/A & N/A \\
& 18 & \textbf{100} & \textbf{219.0} & 0 & N/A & 0 & N/A & 0 & N/A & N/A & N/A \\
\midrule
\multirow{6}{*}{\rotatebox[origin=c]{90}{\textbf{Shelf}}} 
& 3  & \textbf{100} & 31.5  & 4 & \textbf{4.1} & 32 & 9.1 & 68 & 23.7 & \textbf{100} & 88.0 \\
& 6  & \textbf{100} & \textbf{64.6} & 0 & N/A & 0 & N/A & 32 & 48.5 & \textbf{100} & 274.3 \\
& 9  & \textbf{100} & \textbf{80.4} & 0 & N/A & 0 & N/A & 8 & 61.6 & 96 & 607.8 \\
& 12 & \textbf{100} & \textbf{131.5} & 0 & N/A & 0 & N/A & 4 & 95.1 & N/A & N/A \\
& 15 & \textbf{92} & \textbf{187.0} & 0 & N/A & 0 & N/A & 0 & N/A & N/A & N/A \\
& 18 & \textbf{84} & \textbf{324.8} & 0 & N/A & 0 & N/A & 0 & N/A & N/A & N/A \\
\bottomrule
\end{tabular}
\caption{Success rate (S, in percentage) and running time (T, in seconds) across all methods. DM is a standard diffusion model trained on feasible trajectories to directly solve MRMP. The reported running time for DGD is measured under parallel execution.}
\label{tab: Success rate (S, in percentage) and running time (T, in seconds) across all methods.}
\end{table}

Table~\ref{tab: Success rate (S, in percentage) and running time (T, in seconds) across all methods.} presents the success rate (S) and running time (T) across all methods. Across all maps and agent counts, DGD \emph{maintains $100\%$ or near-$100\%$ success}, even in the most challenging Room and Shelf settings with 18 robots. This highlights the strong generalization and feasibility enforcement of our method. DM fails to generate feasible trajectories in all but the easiest scenarios (3 robots). It yields zero success in nearly all settings with 6 or more robots, confirming its limited capacity to handle hard constraint satisfaction. While SMD occasionally matches DGD in success rate, it does so at significantly higher computational cost. For example, in the Basic map with 9 robots, both DGD and SMD reach $100\%$ success, but DGD requires $17.4$s while SMD takes over $500$s. DGD provides a favorable balance between running time and feasibility.

Table~\ref{tab: Path length and acceleration across all methods.} reports the average path length (P) and acceleration (A) across different map types and robot counts. DGD consistently exhibits longer trajectories and higher acceleration compared to the baselines. This behavior is expected, as DGD is the only method capable of handling a broader range of challenging instances, particularly those involving dense environments and large robot teams. These instances naturally require longer paths due to more complex navigation requirements. The increased acceleration values are primarily caused by abrupt transitions between convex regions. Our current framework does not explicitly regulate these transitions, which can result in sharp velocity changes. Incorporating transition-aware smoothing or directly optimizing acceleration during inter-region movement is a promising direction for future work.

\begin{table}[t]
\centering
\small
\setlength{\tabcolsep}{2pt}
\begin{tabular}{lc|cc|cc|cc|cc|cc}
\toprule
Map & Robots 
& \multicolumn{2}{c|}{DGD} 
& \multicolumn{2}{c|}{DM} 
& \multicolumn{2}{c|}{MPD} 
& \multicolumn{2}{c|}{MMD} 
& \multicolumn{2}{c}{SMD} \\
\cmidrule(lr){3-4} \cmidrule(lr){5-6} \cmidrule(lr){7-8} \cmidrule(lr){9-10} \cmidrule(lr){11-12}
& & P & A & P & A & P & A & P & A & P & A \\
\midrule
\multirow{6}{*}{\rotatebox[origin=c]{90}{\textbf{Basic}}} 
& 3  & 1.16 & 0.008 & 1.11 & 0.005 & 1.11 & 0.006 & 1.12 & 0.006 & 1.10 & 0.005 \\
& 6  & 1.21 & 0.011 & N/A & N/A & 1.12 & 0.006 & 1.12 & 0.006 & 1.10 & 0.006 \\
& 9  & 1.26 & 0.013 & N/A & N/A & 1.08 & 0.008 & 1.12 & 0.006 & 1.12 & 0.006 \\
& 12 & 1.28 & 0.014 & N/A & N/A & 1.10 & 0.009 & 1.12 & 0.006 & N/A & N/A \\
& 15 & 1.31 & 0.014 & N/A & N/A & 1.09 & 0.009 & 1.12 & 0.006 & N/A & N/A \\
& 18 & 1.33 & 0.015 & N/A & N/A & 1.09 & 0.008 & 1.13 & 0.006 & N/A & N/A \\
\midrule
\multirow{6}{*}{\rotatebox[origin=c]{90}{\textbf{Dense}}} 
& 3  & 1.17 & 0.007 & N/A & N/A & 1.14 & 0.008 & 1.10 & 0.006 & 1.12 & 0.006 \\
& 6  & 1.26 & 0.012 & N/A & N/A & 1.11 & 0.008 & 1.15 & 0.006 & 1.13 & 0.007 \\
& 9  & 1.30 & 0.014 & N/A & N/A & N/A & N/A & 1.16 & 0.006 & 1.13 & 0.008 \\
& 12 & 1.35 & 0.016 & N/A & N/A & N/A & N/A & 1.15 & 0.006 & N/A & N/A \\
& 15 & 1.35 & 0.017 & N/A & N/A & N/A & N/A & 1.17 & 0.007 & N/A & N/A \\
& 18 & 1.37 & 0.017 & N/A & N/A & N/A & N/A & 1.18 & 0.007 & N/A & N/A \\
\midrule
\multirow{6}{*}{\rotatebox[origin=c]{90}{\textbf{Room}}} 
& 3  & 1.36 & 0.008 & 1.12 & 0.006 & 1.12 & 0.007 & 1.22 & 0.007 & 1.17 & 0.008 \\
& 6  & 1.44 & 0.013 & N/A & N/A & N/A & N/A & 1.16 & 0.007 & 1.19 & 0.007 \\
& 9  & 1.49 & 0.014 & N/A & N/A & N/A & N/A & 1.15 & 0.007 & 1.22 & 0.008 \\
& 12 & 1.50 & 0.015 & N/A & N/A & N/A & N/A & 1.14 & 0.008 & N/A & N/A \\
& 15 & 1.52 & 0.016 & N/A & N/A & N/A & N/A & N/A & N/A & N/A & N/A \\
& 18 & 1.55 & 0.016 & N/A & N/A & N/A & N/A & N/A & N/A & N/A & N/A \\
\midrule
\multirow{6}{*}{\rotatebox[origin=c]{90}{\textbf{Shelf}}} 
& 3  & 1.27 & 0.008 & 1.13 & 0.005 & 1.15 & 0.007 & 1.16 & 0.006 & 1.14 & 0.007 \\
& 6  & 1.30 & 0.011 & N/A & N/A & N/A & N/A & 1.18 & 0.007 & 1.16 & 0.008 \\
& 9  & 1.31 & 0.013 & N/A & N/A & N/A & N/A & 1.16 & 0.007 & 1.17 & 0.009 \\
& 12 & 1.33 & 0.013 & N/A & N/A & N/A & N/A & 1.22 & 0.008 & N/A & N/A \\
& 15 & 1.36 & 0.014 & N/A & N/A & N/A & N/A & N/A & N/A & N/A & N/A \\
& 18 & 1.38 & 0.015 & N/A & N/A & N/A & N/A & N/A & N/A & N/A & N/A \\
\bottomrule
\end{tabular}
\caption{Path length(P) and acceleration (A) across all methods.}
\label{tab: Path length and acceleration across all methods.}
\end{table}

\subsection{Example Trajectories for All Methods}
We visualize example trajectories for each method on 18-robot instances, including only those where the method produced feasible solutions.
Figure~\ref{fig:Representative trajectories generated by DGD across different map types.} shows the representative trajectories generated by DGD in 18 robot instances across the four types of maps. Figure~\ref{fig:Representative trajectories generated by MMD on basic and dense maps.} presents the corresponding results for MMD on basic and dense maps. For visualization purposes, we apply the Savitzky–Golay filter, which is widely used in multi-robot motion planning to smooth trajectories~\cite{shaoul2024multi}. 
\begin{figure}[t]
  \centering
  \begin{subfigure}[t]{0.23\textwidth}
    \centering
    \includegraphics[width=\textwidth]{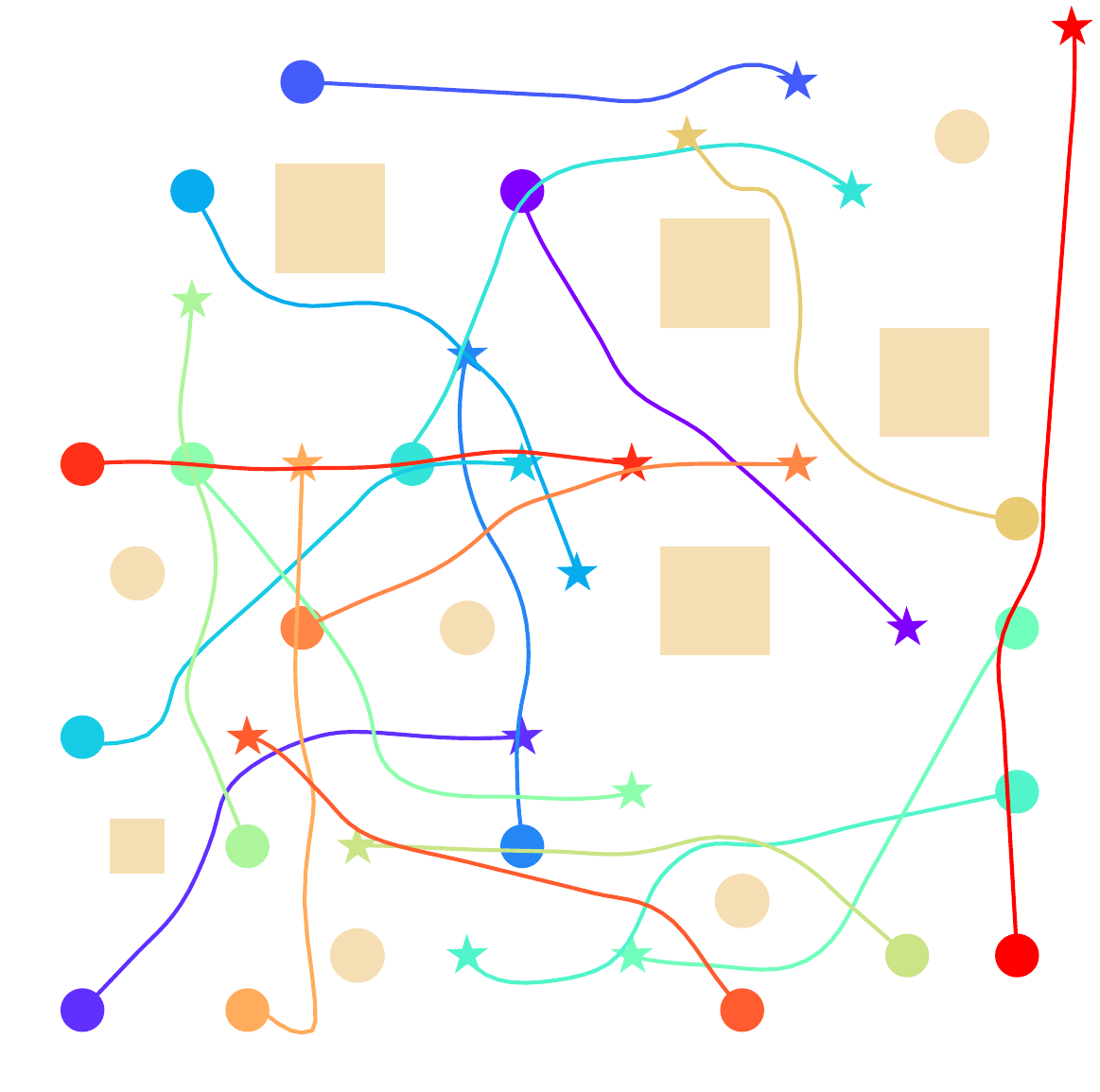}
    \caption{Basic Maps.}
    \label{fig: DGD Trajectories for Basic Maps}
  \end{subfigure}
  \hfill
  \begin{subfigure}[t]{0.23\textwidth}
    \centering
    \includegraphics[width=\textwidth]{Figures/dense_18.pdf}
    \caption{Dense Maps.}
    \label{fig: DGD Trajectories for Dense Maps}
  \end{subfigure}
  \hfill
  \begin{subfigure}[t]{0.23\textwidth}
    \centering
    \includegraphics[width=\textwidth]{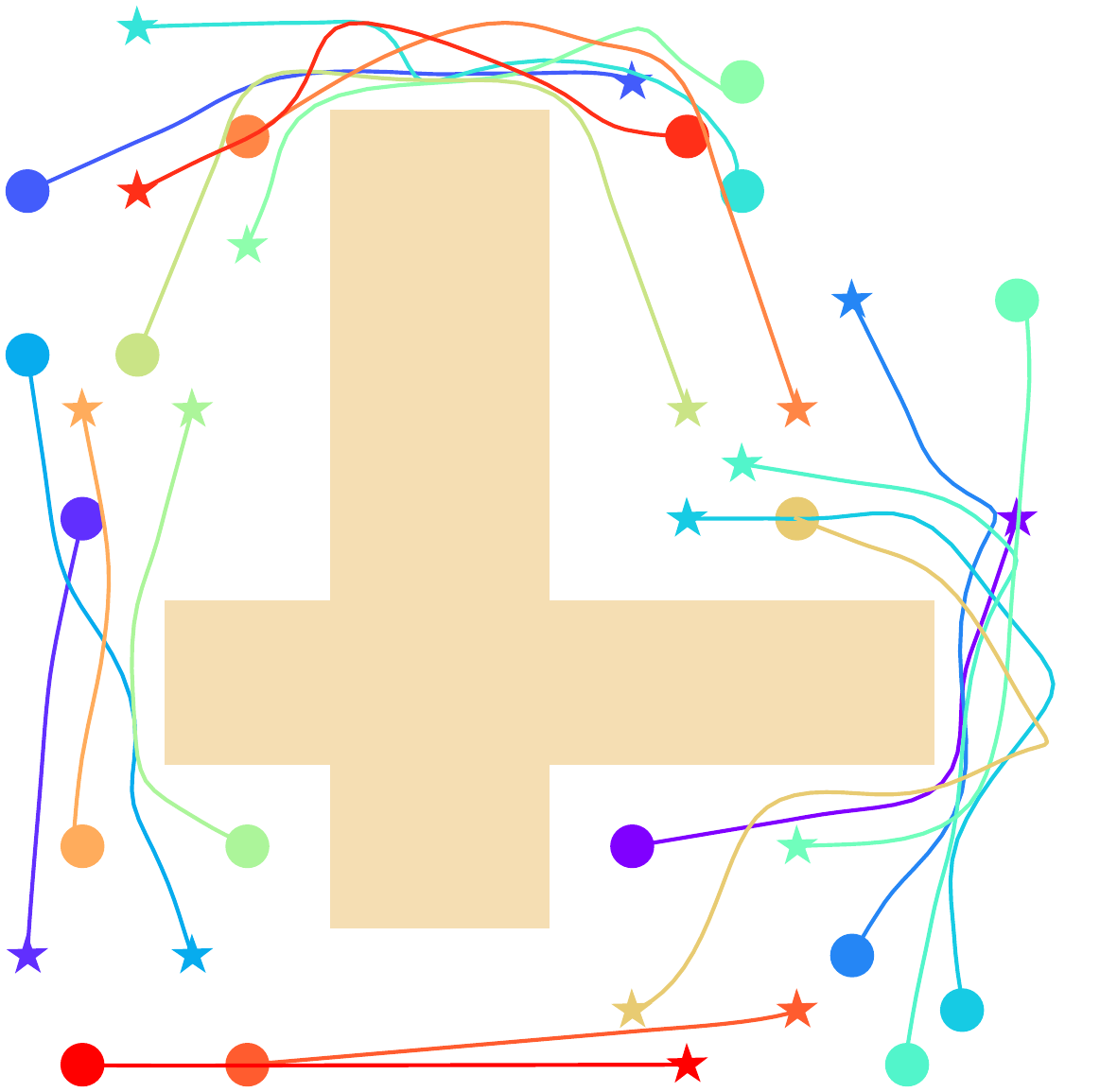}
    \caption{Room Maps.}
    \label{fig: DGD Trajectories for Room Maps}
  \end{subfigure}
  \hfill
  \begin{subfigure}[t]{0.23\textwidth}
    \centering
    \includegraphics[width=\textwidth]{Figures/shelf_18.pdf}
    \caption{Shelf Maps.}
    \label{fig: DGD Trajectories for Shelf Maps}
  \end{subfigure}
  \caption{Representative trajectories generated by DGD across different map types.}
  \label{fig:Representative trajectories generated by DGD across different map types.}
\end{figure}

\begin{figure}[t]
  \centering
  \begin{subfigure}[t]{0.23\textwidth}
    \centering
    \includegraphics[width=\textwidth]{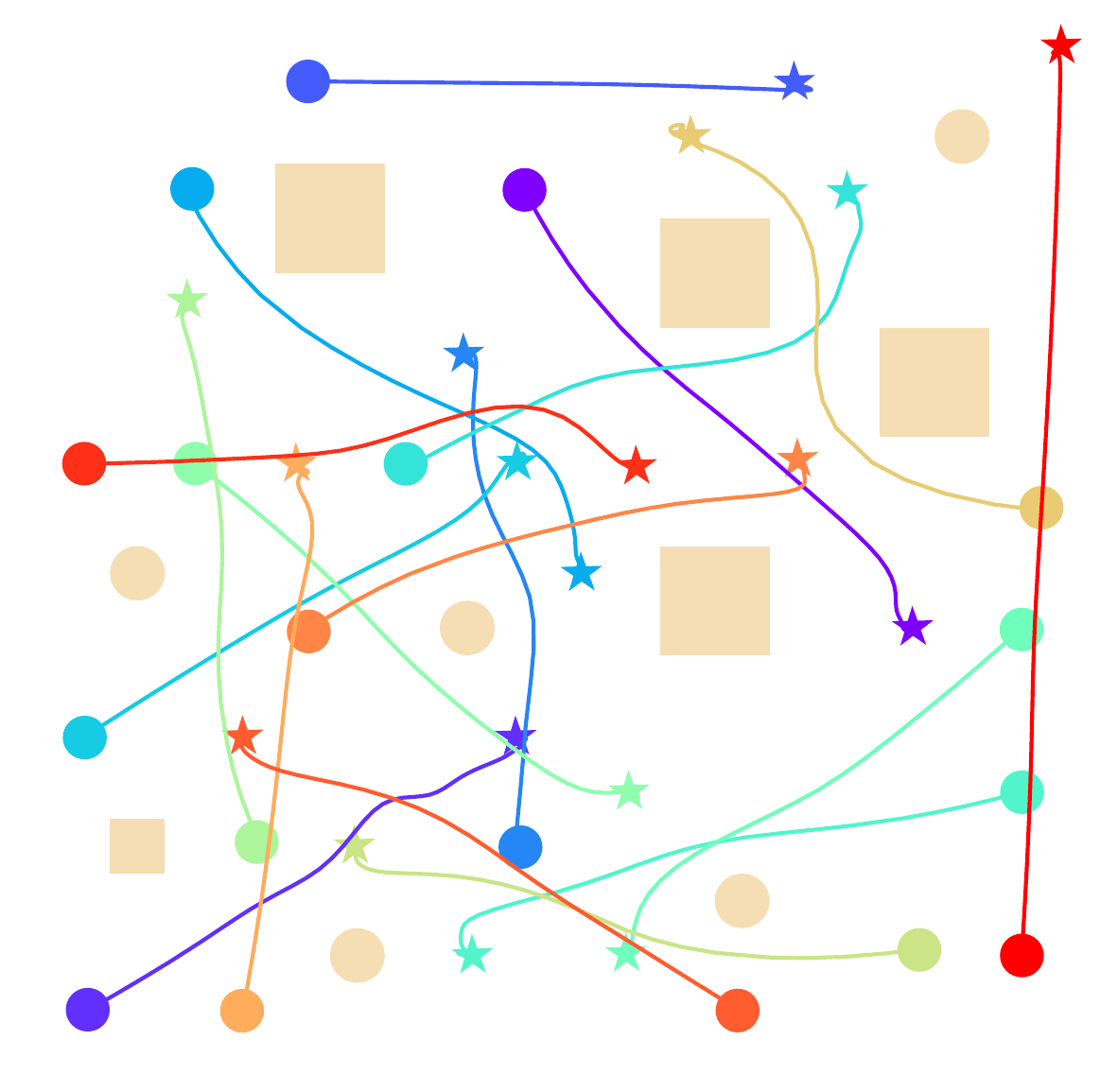}
    \caption{Basic Maps.}
    \label{fig: MMD Trajectories for Basic Maps}
  \end{subfigure}
  \hspace{8pt}
  \begin{subfigure}[t]{0.23\textwidth}
    \centering
    \includegraphics[width=\textwidth]{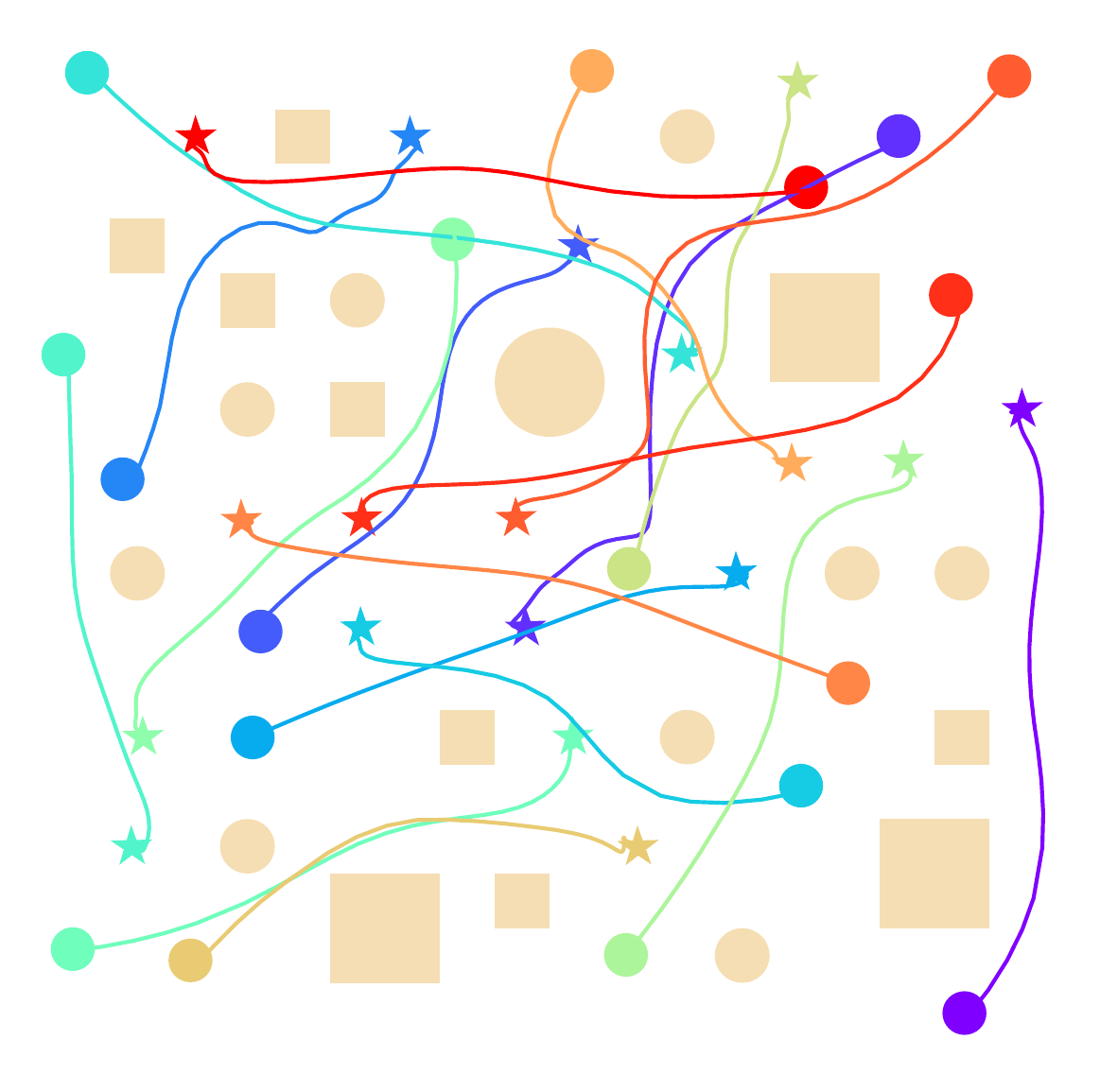}
    \caption{Dense Maps.}
    \label{fig: MMD Trajectories for Dense Maps}
  \end{subfigure}
  \caption{Representative trajectories generated by MMD on basic and dense maps.}
  \label{fig:Representative trajectories generated by MMD on basic and dense maps.}
\end{figure}

We also present the full version of the results on the large map in Figure~\ref{fig:Trajectories generated by DGD on large maps}.
\begin{figure}[t]
    \centering
  \includegraphics[width=0.8\textwidth]{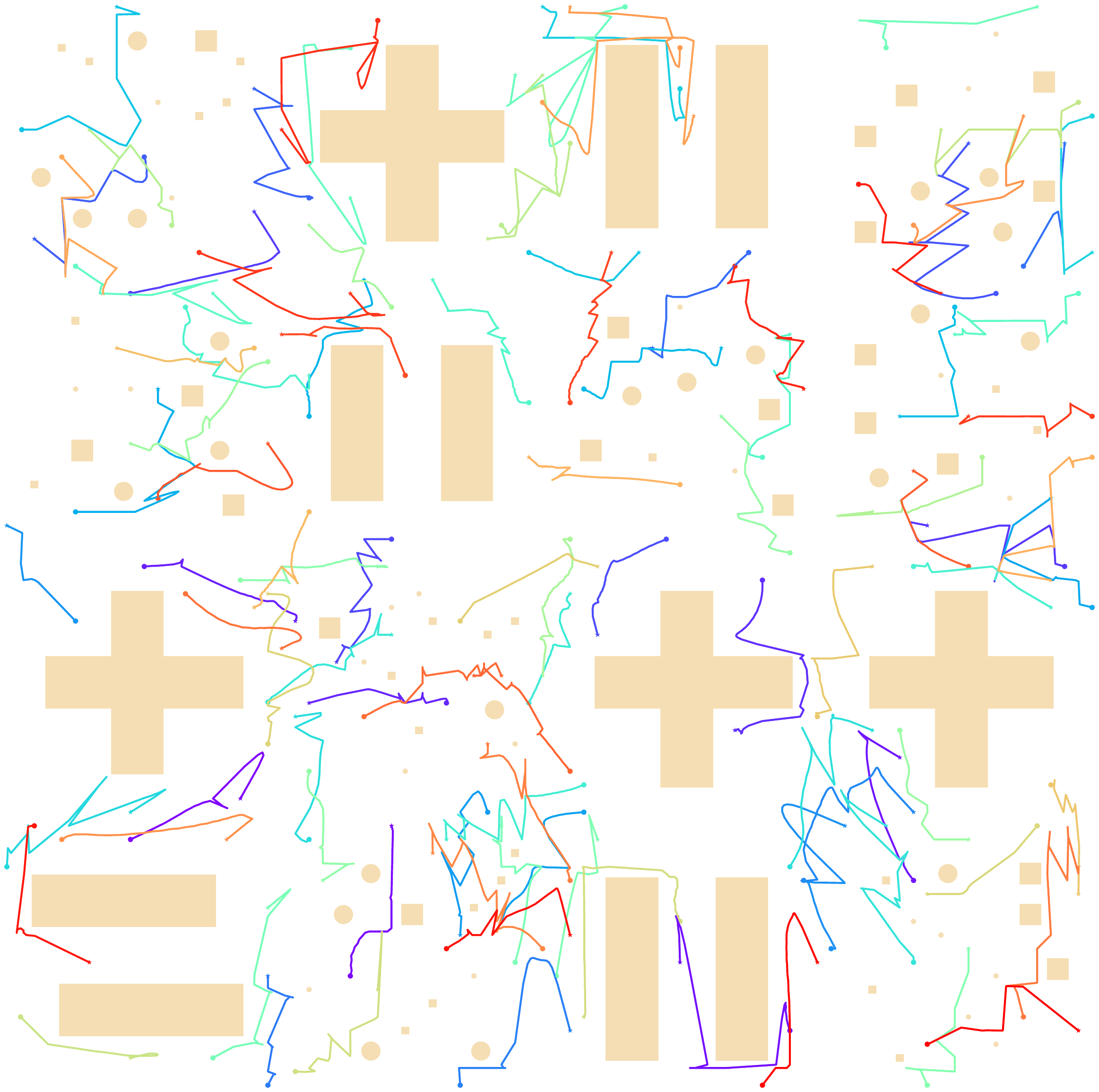}
  \caption{Trajectories generated by DGD on large maps.}
  \label{fig:Trajectories generated by DGD on large maps}
\end{figure}

\section{Implementation Details}
\label{app: Implementation Details}

\textbf{Software}: The software used for experiments is Rocky Linux release 8.9, Python 3.8, Cuda 12.2, and PyTorch 2.1.2.

\textbf{Hardware}: For each of our experiments, we used the NVIDIA RTX A6000 GPU.

\subsection{MRMP Details}
In our experiments, the size of each local map was $2 \times 2$ units. The robot radius was set to 0.04 units in all cases, except for the large-scale map, where it was 0.005 units. The obstacle sizes in basic and dense maps varied between 0.05 and 0.1 units. 

\subsection{Training Details}
Our implementation builds upon the official code of \citet{shaoul2024multi}, \citet{li2021eecbs}, and \citet{ivanfratric2025polypartition}, with modifications to accommodate our specific requirements. Since MMD performs well in generating collision-free trajectories in obstacle-free environments, we leverage it to construct feasible trajectory datasets. Specifically, we first train MMD on single-robot motion planning tasks and then use it to generate multi-robot trajectories for training. Table~\ref{table: hyperparams for Training} summarizes the hyperparameters used during training.
\begin{table}[t]
\centering
\begin{tabular}{cc}
\hline
HyperParameters         & Value \\ \hline
Diffusion Sampling Step &   25    \\
Learning Rate           &   3e-4   \\
Batch Size              &   64   \\
Optimizer               &   Adam    \\ \hline
\end{tabular}
\caption{Hyperparameters for Training in Experiments.}
\label{table: hyperparams for Training}
\end{table}

\end{document}